\documentclass{article} % For LaTeX2e

\usepackage{caption}
\usepackage{subcaption}
\usepackage{jmlr2e}
\usepackage{url}
\usepackage{drago}
\usepackage[on]{editing}
\usepackage{courier}

\title{Causal Algorithmic Recourse: Foundations and Methods}

\author{
    \!\!Drago Ple\v{c}ko \\
    \addr{
        Department of Statistics \& Data Science\\
        UCLA\\
        Los Angeles, CA 90095, USA \\
        \texttt{drago@stat.ucla.edu}
    }
    \AND
    Collin Wang \\ Elias Bareinboim \\
    \addr{
        Department of Computer Science\\
        Columbia University\\
        New York, NY 10027, USA \\
        \texttt{\{clw2180, eb3304\}@columbia.edu}
    }
}

\usepackage{selectp}
\begin{document}
% \treport{25}{May}{2026}
\maketitle

\begin{abstract}
The trustworthiness of AI decision-making systems is increasingly important. A key feature of such systems is the ability to provide recommendations for how an individual may reverse a negative decision, a problem known as algorithmic recourse. Existing approaches treat recourse outcomes as counterfactuals of a fixed unit, ignoring that real-world recourse involves repeated decisions on the same individual under possibly different latent conditions. We develop a causal framework that models recourse as a process over pre- and post-intervention outcomes, allowing for partial stability and resampling of latent variables. We introduce post-recourse stability conditions that enable reasoning about recourse from observational data alone, and develop a copula-based algorithm for inferring the effects of recourse under these conditions. 
For settings where paired observations of the same individual before and after intervention are available (called \emph{recourse data}), we develop methods for inferring copula parameters and performing goodness-of-fit testing. When the copula model is rejected, we provide a distribution-free algorithm for learning recourse effects directly from recourse data. We demonstrate the value of the proposed methods on real and semi-synthetic datasets.
\end{abstract}

\section{Introduction}
Decision-making systems based on machine learning (ML) are being increasingly deployed in real-world settings with far-reaching implications to individuals and society more broadly. This includes hiring decisions, university admissions, law enforcement, credit lending, health care access, and finance \citep{khandani2010consumer,mahoney2007method,brennan2009evaluating}. With an ever larger number of decisions once made by humans now delegated to automated systems, increasing attention has been given to making AI systems trustworthy -- explainable, robust, and fair.

A hallmark feature expected from trustworthy AI systems is the ability to provide recommendations for individuals who wish to reverse a negative decision obtained from such a system. This challenge has been studied in the literature under the rubric of \textit{algorithmic recourse}. For example, an individual who obtained a negative decision when applying for a bank loan may be able to successfully reapply after increasing his/her savings \citep{caglayan2022does}; similarly, an individual who was declined surgical treatment due to increased risk of complications may wish to adjust his/her behavior and dietary habits and seek treatment again in hope of a positive decision \citep{arbous2001mortality}; further examples of such recourse settings are numerous. For automated systems that make decisions in ways not necessarily understandable by humans, providing recourse recommendations to individuals presents an important methodological challenge.

In this paper, we will develop a causal approach to handling questions of algorithmic recourse. To intuitively ground some of the key developments in the remainder of the paper, we begin by discussing a simple example:\phantom{texttexttexttexttexttexttexttexttexttexttexttexttexttexttexttexttext}\vspace{-0.15in}
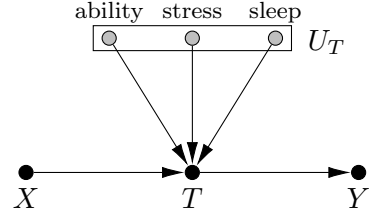
\begin{wrapfigure}{r}{0.4\textwidth}
    \centering
    \vspace{0.1in}
    \begin{tikzpicture}[SCM,scale = 1.1]
        \pgfsetarrows{latex-latex};
            \newcommand{\xshift}{2}
            \newcommand{\yshift}{1.9}

            \node(ut1) at (1.5 * \xshift, 0.85 * \yshift)[label=above:{\footnotesize ability}, point, fill=gray!50];
            \node(ut2) at (2 * \xshift, 0.85 * \yshift)[label={[yshift=1.5pt]above:{\footnotesize stress}}, point, fill=gray!50];
            \node(ut3) at (2.5 * \xshift, 0.85 * \yshift)[label=above:{\footnotesize sleep}, point, fill=gray!50];
            \draw ($(ut1.west) + (-0.1,0.15)$) rectangle ($(ut3.east) + (0.1,-0.15)$) node[right, xshift=0.1cm, yshift=0.1cm] {$U_T$};
            
            \node(x) at (1 * \xshift, 0 * \yshift)[label=below:{$X$}, point];
            \node(t) at (2 * \xshift, 0 * \yshift)[label=below:{$T$}, point];
            \node(y) at (3 * \xshift, 0 * \yshift)[label=below:{$Y$}, point];
            
            \path (x) edge (t);
            \path (t) edge (y);
            \path (ut1) edge (t);
            \path (ut2) edge (t);
            \path (ut3) edge (t);
        \end{tikzpicture}
        \caption{Causal diagram for Ex.~\ref{ex:intro}.}
         \vspace{-0.3in}
        \label{fig:intro-ex}
\end{wrapfigure}
\begin{example}[Exam Repetition] \label{ex:intro}
    Students at a university are taking a mandatory exam. For preparation, students are allowed to attend office hours held by the teaching assistants (variable $X$, where $x_0$ represents that the student attended office hours, and $x_1$ if she/he did not). The test scores of students are denoted by $T$, a numeric value in the reals $\mathbbm{R}$. The outcome $Y$ of whether a student passed the exam is a simple threshold operation, $Y = \mathbbm{1}(T \geq t)$, where $t$ is fixed by the university every year. 

    A graphical description of how variables affect each other in this setting is shown in Fig.~\ref{fig:intro-ex}. In particular, latent (unobserved) variables that affect the student's test score $T$ are drawn in gray, including the student's ability, amount of stress in the last 2 weeks, and amount of sleep in the last 2 weeks.

    Consider now a student who did not attend office hours ($X = x_0$) and obtained a test score of $t' < t$, thus failing the exam. The student is told to re-take the exam after two weeks, and is encouraged to attend office hours, i.e., set $X = x_1$ instead of $x_0$ in the coming week.

    Let $u_T$ denote the latent noise variables that determined the test outcome $T$ for the student, together with $x_0$. In the first instance, we observed the student's test score $T$, which can be written as a function of the $X = x_0$ and the latent $u_T$, i.e., $T \gets f_T(x_0, u_T)$, and $T$ is realized as $t' = f_T(x_0, u_T)$. After two weeks, the student followed the recommendation of attending office hours, and now $X = x_1$. However, the latent noise variables that consist of ability, stress, and sleep levels may be different for the second attempt at taking the exam. Formally, the latent $u_T$ for the first attempt may differ from $u^*_T$ for the second attempt. Thus, in the second attempt, the evaluation of $f_T$ will be $f_T(x_1, u^*_T)$, and the student will attain a new score $t^*$ possibly different from $t'$. Importantly, the latent $u^*_T$ may be related to $u_T$, since these two latent variables relate to the same individual.
\end{example}
Three key observations ensue from the above example:
\begin{enumerate}[label=(\arabic*)]
    \item In recourse settings, we may be able to observe two (or possibly more) outcomes, typically pre- and post-intervention. These outcomes, however, are for the \textit{same individual}, and thus the noise variables $u_T$ and $u_T^*$ affecting them may share information,
    \item The pre- and post-intervention outcomes, written as $f_T(x_0, u_T)$, and $f_T(x_1, u^*_T)$, may therefore be related. For instance, one could reasonably expect that some part of the latent $u_T$ remains the same (e.g., student's ability),
    \item At the same time, however, other parts of the latent $u_T$ may change. For instance, the student's sleep and stress levels may possibly differ between the two exam attempts.
\end{enumerate}
Therefore, when observing pre- and post-recourse outcomes (realizations of $f_T(x_0, u_T)$, $f_T(x_1, u^*_T)$), we need to take into account that $u_T$ and $u^*_T$ may share information but are not necessarily identical. In other words, the process of repeating actions in the real world adds inherent structural uncertainty, and the latent variables $u_T$ may be perturbed to $u^*_T$. The ability to jointly observe outcomes of the same variable for a single individual will be a key feature of recourse data and motivates the methodological developments described in the sequel.

\subsection{Relationship to Previous Literature}
We now describe how the approach we present relates to previous works on algorithmic recourse. In particular, we explain the relation to (i) non-causal approaches to recourse; (ii) existing causal approaches; and (iii) the literature on combining observational and experimental data.

\paragraph{Contrastive Explanations.}
The approach of contrastive, or counterfactual, explanations seeks to find an alternative set of covariate values, close to the true ones, that would have resulted in a different decision for the individual \citep{wachter2017counterfactual}. Various notions of distance between features and the cost of manipulating them have been explored in the literature, in an attempt to measure how difficult it may be for an individual to change their attributes \citep[Sec.~3.1]{karimi2022survey}. If there is no causal structure among the covariates, and each feature may be manipulated by the individual, such explanations may also lead to recommendations on how the individual may change their outcome \citep{wachter2017counterfactual, joshi2019towards, sharma2020certifai, ustun2019actionable}. A commonly used assumption is the \textit{independently manipulable features} (IMF) assumption, which basically rules out any causal effects between the covariates used for prediction.
However, methods that ignore the fact that actions taken by the individual may affect other decision-related characteristics may lead to suboptimal or infeasible recommendations \citep{barocas2020hidden, venkatasubramanian2020philosophical}. We discuss an explicit example of this in Sec.~\ref{sec:IMF}.

\paragraph{Causal Algorithmic Recourse.}
Causal recourse methods that explicitly take into account the underlying causal model of the environment have been proposed \citep{mahajan2019preserving, karimi2021interventions, kugelgen2022fairrecourse} in order to improve upon the possibly rigid IMF assumption. However, some of these methods require the full specification of the generating model of reality, in terms of a structural causal model (SCM, \cite{pearl:2k}), and their assumptions are therefore too stringent for most practical applications \citep{bareinboim2020on}. 
The work most related to ours is that of \citet{karimi2020probabilistic}, which assumes the availability of a causal diagram and data to perform inference, corresponding to the standard setting in the causality literature, and also the setting considered in this paper.
The authors offer two solutions. They either make additive noise assumptions that render the joint distribution over counterfactual outcomes identifiable, or they consider an alternative method that avoids considering the joint, at the expense of recourse recommendations not being covariate-specific. 
Importantly, the formulation of recourse offered in this work considers the recourse outcomes to be \textit{exact counterfactuals}, and does not consider the implications of the unit's pre- and post-recourse latent variables possibly being different, as described in Ex.~\ref{ex:intro}. Therefore, this prior work can be seen as a special case of the framework we propose, in which the pre- and post-recourse units are identical. 

\paragraph{Combining Observational and Experimental Data.} The final line of work related to ours is that of identifying interventional \citep{lee2020general} or counterfactual \citep{correa2021nested} distributions from combinations of experimental and observational data. In the setting of recourse, however, there is a distinctive feature of observational and experimental samples being related, as they correspond to the same individual, as discussed in Ex.~\ref{ex:intro}. This makes the setting unique, and any off-the-shelf inference ideas that just combine observational and experimental data will not be sufficient. In fact, our work is the first to formalize an inference problem of this kind, and the first to consider the concept of learning from \textit{recourse data}.

After demonstrating how our work fits into the broader context, we can list the specific contributions found in this paper:

\begin{enumerate}[label=(\arabic*)]
    \item We develop a formulation of algorithmic recourse based on structural causal models that takes into account the inherent variability in the latent variables described in Ex.~\ref{ex:intro} (Def.~\ref{def:new-recourse}). We define the so-called post-recourse stability (PRS) conditions that guarantee that, for a fixed set of causal parents, distributions over latent variables pre- and post-recourse coincide. We prove that recourse effects may be inferred from observational data under PRS (Thm.~\ref{thm:prs-implies-msc}),
    
    \item We develop an algorithm (Alg.~\ref{algo:frank-obs}) for non-parametrically inferring effects of recourse actions under PRS, using observational data and the assumption that the joint distribution over factual and recourse outcomes is modeled by Frank's copula \citep{frank1979simultaneous}. 
    %Since the parameter of the copula cannot be inferred from observational data, our procedure is a sensitivity-type of analysis, with the sensitivity parameter being Kendall's $\tau$ correlation coefficient \citep{kendall1948rank}. 
    
    \item We introduce a novel causal notion of recourse data (Def.~\ref{def:recourse-data}) in which we have access to an observational sample and an interventional sample for the same individual. We define a new class of learning problems based on recourse data (Def.~\ref{def:recourse-learning}). We develop an algorithm (Alg.~\ref{algo:frank-recourse}) for inferring the effects of recourse actions under PRS, using observational and recourse data. Furthermore, we propose a goodness-of-fit hypothesis test that can be performed to falsify the assumption that Frank's copula is an appropriate modeling choice.
    
    \item We develop an algorithm (Alg.~\ref{algo:copula-free-learning}) for inferring effects of recourse actions from observational and recourse data for the case when PRS does not hold, and provide theoretical guarantees for the algorithm in the linear case (Thm.~\ref{thm:linear-recourse}).
\end{enumerate}
In Fig.~\ref{fig:recourse-overview}, we provide a visual summary of the causal recourse framework introduced in this paper.
In the following section, we cover some important preliminary notions for our discussion. Readers familiar with the language of structural causality may wish to go straight to Sec.~\ref{sec:new-recourse}.
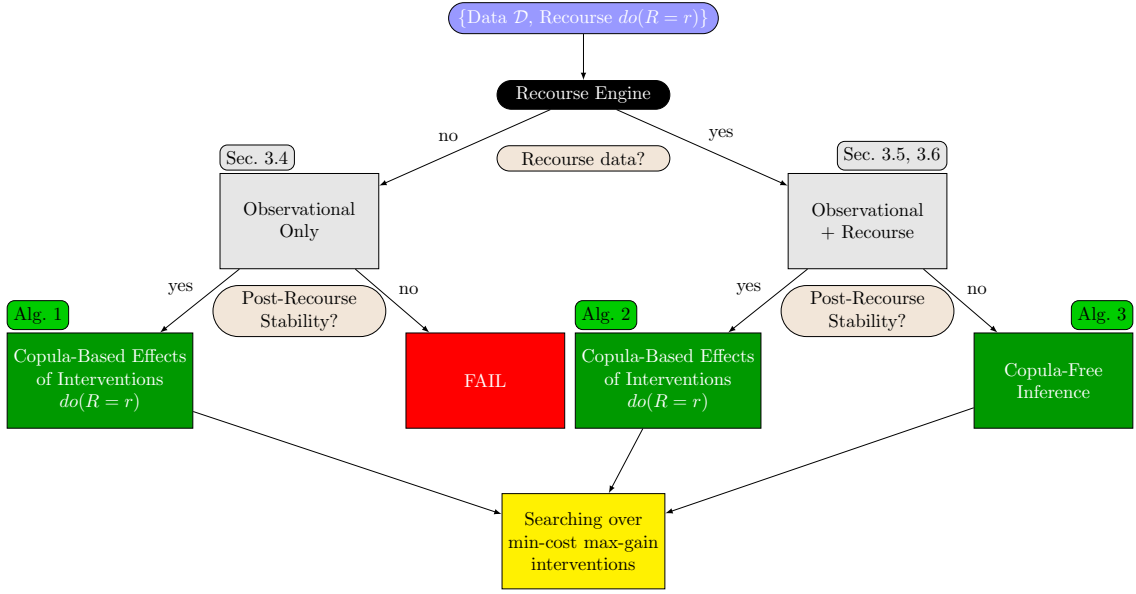
\begin{figure}
    \centering
    \scalebox{0.7}{
    \begin{tikzpicture}[
    node distance=1.2cm and 1.5cm,
    box/.style={rectangle, draw, minimum height=1.8cm, minimum width=3cm, align=center, fill=gray!20},
    engine/.style={rounded rectangle, draw, fill=black, text=white, minimum width=3.5cm, align=center},
    decision/.style={rounded rectangle, draw, fill=brown!20, text=black, minimum width=3.5cm, align=center},
    fail/.style={rectangle, draw, fill=red, text=white, minimum height=2cm, minimum width=3.5cm, align=center},
    arrow/.style={-latex}
]

\node[engine] (recourse) {Recourse Engine};
\node[engine, above of = recourse, fill=blue!40,yshift=0.25cm] (input) {\{Data $\mathcal{D}$, Recourse $do(R=r)$\}};
\node[decision, below of=recourse] (rec-data) {Recourse data?}; 
\node[box, below left = of recourse, xshift=-1cm] (obsOnly) {Observational\\ Only};
\node[box, below right = of recourse, xshift=1cm] (obsRecourse) {Observational\\ + Recourse};
\node[box, below left = of obsOnly, xshift=1cm, fill=green!60!black, text=white] (copula1) {Copula-Based Effects\\ of Interventions \\ $do(R = r)$};
\node[decision, below of=obsOnly, yshift=-0.5cm] (msc1) {Post-Recourse\\ Stability?}; 
\node[box, below right = of obsOnly, fill=red, text=white, xshift=-1cm] (fail) {FAIL};
\node[box, below left = of obsRecourse, fill=green!60!black, text=white, xshift=1cm] (copula2) {Copula-Based Effects \\ of Interventions \\ $do(R = r)$};
\node[decision, below of=obsRecourse, yshift=-0.5cm] (msc2) {Post-Recourse\\ Stability?}; 
\node[box, below right = of obsRecourse, fill=green!60!black, text=white, xshift=-1cm] (free) {Copula-Free\\ Inference};
\node[box, below = 7.25cm of recourse, fill=yellow] (search) {Searching over\\ min-cost max-gain\\ interventions};

\node[draw, rounded corners, anchor=south west, yshift=0.05cm, fill=gray!20] at (obsOnly.north west) {Sec.~\ref{sec:msc-obs}};
\node[draw, rounded corners, anchor=south east, yshift=0.05cm, fill=gray!20] at (obsRecourse.north east) {Sec.~\ref{sec:msc-rec}, \ref{sec:mic-rec}};

% add labels
\node[draw, rounded corners, anchor=south west, yshift=0.05cm, fill=green!80!black] at (copula1.north west) {Alg.~\ref{algo:frank-obs}};
\node[draw, rounded corners, anchor=south west, yshift=0.05cm, fill=green!80!black] at (copula2.north west) {Alg.~\ref{algo:frank-recourse}};
\node[draw, rounded corners, anchor=south east, yshift=0.05cm, fill=green!80!black] at (free.north east) {Alg.~\ref{algo:copula-free-learning}};

\draw[arrow] (input) -- (recourse);
\draw[arrow] (recourse) -- node[above left] {no} (obsOnly);
\draw[arrow] (recourse) -- node[above right] {yes} (obsRecourse);
\draw[arrow] (obsOnly) -- node[above left] {yes} (copula1);
\draw[arrow] (obsOnly) -- node[above right] {no} (fail);
\draw[arrow] (obsRecourse) -- node[above left] {yes} (copula2);
\draw[arrow] (obsRecourse) -- node[above right] {no} (free);
\draw[arrow] (free) -- (search);
\draw[arrow] (copula1) -- (search);
\draw[arrow] (copula2) -- (search);

\end{tikzpicture} 
    }
    \caption{Overview of the framework proposed in this paper.}
    \label{fig:recourse-overview}
\end{figure}
\section{Preliminaries}
We use the language of structural causal models (SCMs) as our basic semantical framework \citep{pearl:2k}. A structural causal model (SCM) is defined as:
\begin{definition}[Structural Causal Model (SCM) \citep{pearl:2k}] \label{def:SCM}
	A structural causal model $\mathcal{M}$ is a 4-tuple $\langle V, U, \mathcal{F}, P(u)\rangle$, where
  \begin{enumerate}
    \item $U$ is a set of exogenous variables, also called background variables, that are determined by factors outside the model;
    \item $V = \lbrace V_1, ..., V_n \rbrace$ is a set of endogenous (observed) variables, that are determined by variables in the model (i.e. by the variables in $U \cup V$);
    \item $\mathcal{F} = \lbrace f_{V_1}, ..., f_{V_n} \rbrace$ is the set of structural functions determining $V$, $v_i \gets f_{V_i}(\pa(v_i), u_i)$, where $\pa(V_i) \subseteq V \setminus V_i$ and $U_i \subseteq U$ are the functional arguments of $f_{V_i}$;
    \item $P(u)$ is a distribution over the exogenous variables $U$.
  \end{enumerate}
\end{definition}
\noindent The assignment mechanisms $\mathcal{F}$ determine how each of the observed variables $V_i$ attains its value, based on other observed variables and the latent variables $U$. Together with the probability distribution $P(u)$ over the exogenous variables $U$, it specifies the entire behavior of the underlying phenomenon. In particular, the SCM also specifies the \textit{observational distribution} of the underlying phenomenon, defined through:
\begin{definition}[Observational Distribution \citep{bareinboim2020on}] \label{def:obs-dist}
An SCM $\mathcal{M}$ that is a 4-tuple $\langle V, U, \mathcal{F}, P(u) \rangle$ induces a joint probability distribution $P(V)$ such that for each $Y \subseteq V$,
\begin{align}
    P^{\mathcal{M}}(y) = \sum_{u} \mathbb{1}\Big(Y(u) = y \Big) P(u),
\end{align} 
where $Y(u)$ is the solution for $Y$ after evaluating $\mathcal{F}$ with $U = u$.
\end{definition}
\noindent A further important notion building on the concept of the SCM is that of a submodel, which is defined next:
\begin{definition}[Submodel \citep{pearl:2k}] \label{def:submodel}
    Let $\mathcal{M}$ be a structural causal model, $X$ a set of variables in $V$, and $x$ a particular value of $X$. A submodel $\mathcal{M}_{x}$ (of $\mathcal{M}$) is a 4-tuple:
    \begin{equation}
        \mathcal{M}_{x} = \langle V, U, \mathcal{F}_{x}, P(u)\rangle
    \end{equation}
    where 
    \begin{equation}
        \mathcal{F}_{x} = \lbrace f_i : V_i \notin X \rbrace \cup \lbrace X \gets x\rbrace,
    \end{equation}
    and all other components are preserved from $\mathcal{M}$. 
\end{definition}
\noindent Building on submodels, we introduce next the
notion of a potential outcome:
\begin{definition}[Potential Outcome / Response \citep{rubin1974estimating, pearl:2k}]\label{def:potentialresponse}
    Let $X$ and $Y$ be two sets of variables in $V$ and $u \in \mathcal{U}$ be a unit. The potential outcome/response $Y_x(u)$ is defined as the solution for $Y$ of the set of equations $\mathcal{F}_x$ evaluated with $U=u$. That is, $Y_x(u)$ denotes the solution of $Y$ in the submodel $\mathcal{M}_x$ of $\mathcal{M}$.
\end{definition}
\noindent In words, $Y_x(u)$ is the value variable $Y$ would take if (possibly contrary to observed facts) $X$ is set to $x$, for a specific unit $u$. In Ex.~\ref{ex:intro}, one could think of the potential outcome of the test score $T$ subject to setting $X = 1$, which would be written $T_{X=1}(u)$. We further define how counterfactual distributions over various possible potential outcomes are computed:
\begin{definition}[Counterfactual Distributions \citep{bareinboim2020on}] \label{def:ctf-dist}
    Consider an SCM $\mathcal{M} = \langle V, U, \mathcal{F}, P(u) \rangle$, and let $Y_1, \dots, Y_k \subset V$, and $X_1, \dots, X_k \subset V$ be subsets of the observables, and let $x_1, \dots, x_k$ be specific values of $X_i$s. Denote by $(Y_i)_{x_i}$ the potential response of variables $Y_i$ when setting $X_i = x_i$. The SCM $\mathcal{M}$ induces a family of joint distributions over counterfactual events $(Y_1)_{x_1}, \dots, (Y_k)_{x_k}$:
    \begin{align} \label{eq:L3def}
        P^{\mathcal{M}}((y_1)_{x_1}, \dots, (y_k)_{x_k}) = \sum_{u} \mathbb{1}\Big( \bigwedge_{i=1}^k (Y_i)_{x_i}(u) = y_i \Big) P(u).
    \end{align}
\end{definition}
\noindent The l.h.s. in Eq.~\ref{eq:L3def} contains variables with different subscripts, which syntactically represent different potential responses (Def.~\ref{def:potentialresponse}), or counterfactual worlds. 
%In words, the equation can be interpreted as follows:
Finally, there is one more prerequisite notion for our discussion. The mechanisms $\mathcal{F}$ and the distribution over the exogenous variables $P(u)$ are almost never observed. However, to perform causal inference, we need a way of encoding assumptions about the underlying SCM. A common way of doing so is through an object called a causal diagram, which is defined next: 
\begin{definition}[Causal Diagram \citep{pearl:2k, bareinboim2020on}] \label{def:diagram}
	Let an SCM $\mathcal{M}$ be a 4-tuple $\langle V, U, \mathcal{F}, P(u)\rangle$. A graph $\mathcal{G}$ is said to be a \textit{causal diagram} (of $\mathcal{M}$) if:
	\begin{enumerate}[label=(\arabic*)]
		\item there is a vertex for every endogenous variable $V_i \in V$,
		\item there is an edge $V_i \to V_j$ if $V_i$ appears as an argument of $f_j \in \mathcal{F}$,
		\item there is a bidirected edge $V_i\dashleftarrow\dasharrow V_j$ if the corresponding $U_i, U_j \subset U$ are correlated or the corresponding functions $f_i, f_j$ share some $U_{ij} \in U$ as an argument.
	\end{enumerate}
\end{definition}
We call $\pa(V_i)$ the set of parents of $V_i$,  while the sets of children $\ch(V_i)$, ancestors $\an(V_i)$, and descendants $\de(V_i)$ are defined analogously.

\paragraph{Prior Recourse Definition.} Recourse actions will generally be indicated by $do(R = r)$, and one may be interested in the potential responses of the outcome under the recourse action for specific units, labeled $\widehat{Y}_{R=r}(u)$. When clear from context, we drop the $R=r$ notation and write only $r$ instead. A previous formulation of algorithmic recourse, proposed in \citep{karimi2020probabilistic}, is given as follows:
\begin{definition}[Recourse Problem \citep{karimi2020probabilistic}] \label{def:recourse-old} Let $R \subset V$ be a subset of the observables on which recourse may be performed. Let $r$ be a fixed value of $R$, and let $t(r)$ be a function of $R = r$. Then, the optimal recourse action for an individual with $V = v$ is given by:
    \begin{align}
        \argmin_{R, r} &\quad  \text{cost(R = r)} \label{eq:cost-term}\\
        \text{subject to}&\quad P(\widehat{Y}_{R=r} = 1 \mid v) \geq t(r). \label{eq:guarantee-term}
    \end{align}
\end{definition}
\noindent The intuition behind the above formulation is simple. The goal is to minimize the cost the individual incurrs from performing recourse (Eq.~\ref{eq:cost-term}) while ensuring that the probability of success after recourse is large enough (Eq.~\ref{eq:guarantee-term}). 

\subsection{Independently Manipulable Features} \label{sec:IMF}
In this section, we first discuss the independently manipulable features (IMF) assumption, frequently used in the literature on algorithmic recourse and counterfactual explanations \citep{wachter2017counterfactual, ustun2019actionable, joshi2019towards}. In particular, the IMF assumption can be represented as a causal diagram, shown in Fig.~\ref{fig:imf-assumed}, where the dots ($\cdots$\!) represent variables $V_2, \dots, V_{k-1}$ between $V_1$ and $V_k$. Each variable $V_1, \dots, V_k$ has an edge to $\widehat{Y}$, but no other edges exist. 
\begin{figure}
    \centering
    \begin{subfigure}{0.24\textwidth}
        \centering
        \begin{tikzpicture}[SCM,scale = 1]
        \pgfsetarrows{latex-latex};
            \newcommand{\xshift}{1.5}
            \newcommand{\yshift}{1.4}
            \node(v1) at (1 * \xshift, 0.5 * \yshift)[label=above:{$V_1$}, point];
            \node(vdots) at (2 * \xshift, 0.5 * \yshift){$\dots$};
            \node(vk) at (3 * \xshift, 0.5 * \yshift)[label=above:{$V_k$}, point];
            \node(yhat) at (2 * \xshift, -0.5 * \yshift)[label=below:{$\widehat{Y}$}, point];
            
            \path (v1) edge (yhat);
            \path (vdots) edge (yhat);
            \path (vk) edge (yhat);
        \end{tikzpicture}
        \caption{IMF in general.}
        \label{fig:imf-assumed}
    \end{subfigure}
    \hfill
    \begin{subfigure}{0.24\textwidth}
    \centering
    \begin{tikzpicture}[SCM,scale = 1]
        \pgfsetarrows{latex-latex};
            \newcommand{\xshift}{1.5}
            \newcommand{\yshift}{1.4}
            \node(v1) at (1 * \xshift, 0.5 * \yshift)[label=above:{$V_1$}, point];
            \node(vk) at (3 * \xshift, 0.5 * \yshift)[label=above:{$V_2$}, point];
            \node(yhat) at (2 * \xshift, -0.5 * \yshift)[label=below:{$\widehat{Y}$}, point];
            
            \path (v1) edge (yhat);
            \path (vk) edge (yhat);
        \end{tikzpicture}
        \caption{IMF in Ex.~\ref{ex:imf-failure}.}
        \label{fig:imf-example}
    \end{subfigure}
    \hfill
    \begin{subfigure}{0.24\textwidth}
        \centering
        \begin{tikzpicture}[SCM,scale = 1]
        \pgfsetarrows{latex-latex};
            \newcommand{\xshift}{1.5}
            \newcommand{\yshift}{1.4}
            \node(v1) at (1 * \xshift, 0.5 * \yshift)[label=above:{$V_1$}, point];
            \node(vk) at (3 * \xshift, 0.5 * \yshift)[label=above:{$V_2$}, point];
            \node(yhat) at (2 * \xshift, -0.5 * \yshift)[label=below:{$\widehat{Y}$}, point];
            
            \path (v1) edge (vk);
            \path (v1) edge (yhat);
            \path (vk) edge (yhat);
        \end{tikzpicture}
        \caption{Diagram of Ex.~\ref{ex:imf-failure}.}
        \label{fig:imf-fails}
    \end{subfigure}
    \hfill
    \begin{subfigure}{0.24\textwidth}
        \centering
        \begin{tikzpicture}[SCM,scale = 1]
        \pgfsetarrows{latex-latex};
            \newcommand{\xshift}{1}
            \newcommand{\yshift}{1}
            \node(v1) at (1 * \xshift, 0.5 * \yshift)[label=above:{$V_1$}, point];
            \node(v2) at (3 * \xshift, 0.5 * \yshift)[label=above:{$V_2$}, point];
            \node(v3) at (1 * \xshift, -0.5 * \yshift)[label=below:{$V_3$}, point];
            \node(v4) at (3 * \xshift, -0.5 * \yshift)[label=below:{$V_4$}, point];
            \node(yhat) at (2 * \xshift, -1 * \yshift)[label=below:{$\widehat{Y}$}, point];
            
            \path (v1) edge (v3);
            \path (v2) edge (v4);
            \path (v1) edge (yhat);
            \path (v2) edge (yhat);
            \path (v3) edge (yhat);
            \path (v4) edge (yhat);
    \end{tikzpicture}
    \caption{Diagram of Ex.~\ref{ex:ctf-vs-int}.}
    \end{subfigure}
    \caption{Independently manipulable feature (IMF) assumption for the general case (a) and for Ex.~\ref{ex:imf-failure} (b); (c) show the true causal diagram for Ex.~\ref{ex:imf-failure}; (d) causal diagram for Ex.~\ref{ex:ctf-vs-int}.}
    \label{fig:imf-graphs}
\end{figure}
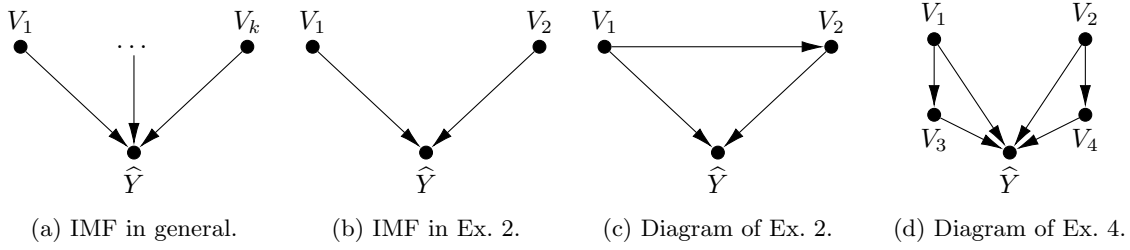
This assumption, often made to simplify inference of recourse recommendations, can have serious implications even in the presence of simple causal structure among the features, and may lead to wrong conclusions as witnessed by the following example:
\begin{example}[IMF Assumption Failure] \label{ex:imf-failure}
    Consider the following SCM $\mathcal{M}$:
    \begin{align}
        V_1 &\gets N(0, 1) \\
        V_2 &\gets \alpha V_1 + N(0, 1) \\
        \widehat{Y} &\gets \mathbb{1} ( V_1 - V_2 \geq 1)
    \end{align}
    where $V_1$ represents the income level, $V_2$ total expenditure, and $\widehat{Y}$ whether a person is granted a cash loan. The causal diagram induced by $\mathcal{M}$ is shown in Fig.~\ref{fig:imf-fails}.
    In words, the total expenditure $V_2$ is a linear function of income $V_1$, and if $\alpha < 1$ the individual (on average) spends less than he/she earns, and spends more otherwise. The loan decision $\widehat Y$ is strictly based on whether the individual's income is larger than their expenditure by a fixed amount of $1$.
    Now, suppose we have an individual with features $(v_1, v_2) = (1, 0.5)$ who has his/her loan declined.

    Suppose that $V_1$ is the only manipulable feature in this scenario, and suppose that a data analyst invokes the IMF assumption, which in this case does not hold due to a causal influence $V_1 \to V_2$. By knowing the mechanism of $\widehat{Y}$ (usually known to the system designer), they would reason as follows. To increase their creditworthiness, the person needs to increase their difference in income vs. expenditure, from the current $v_1 - v_2 = 0.5$ to $1$. Under IMF, a sufficient intervention would be to add $0.5$ to $V_1$, i.e., do the recourse action $v_1 = 0.5 \to v_1 = 1$. 

    The IMF assumption would entirely disregard the fact that increased income may also be linked with increased expenditure, as described by the $\alpha$ coefficient; the IMF assumption implicitly assumes that $\alpha = 0$. However, in the SCM $\mathcal{M}$, the minimal necessary intervention when taking into account the causal structure $V_1 \to V_2$ would be equal to
    \begin{align}
        \argmin_{\delta} v_1 + \delta - (v_2 - \alpha \delta) \geq 1, \text{ for }v_1 = 1, v_2 = 0.5,
    \end{align}
    with the solution $\delta_{min} = \frac{1}{2(1-\alpha)}$ for $\alpha < 1$ and no solution for $\alpha \geq 1$, yielding an entirely different insight from what was obtained based on the IMF assumption.
\end{example}
This example highlights that, even in the basic linear setting, it can be rather important to consider the causal structure among variables. Therefore, the IMF assumption, commonly invoked in the literature, may have significant drawbacks.

\subsection{Hardness of Recourse -- Joint Counterfactual Distributions}
For the discussion in this section, suppose for simplicity that the set of covariates under recourse $R$ is the set of all root nodes in the causal diagram $\mathcal{G}$, and that the decision mechanism $f_{\widehat{y}}$ of $\widehat{Y}$ is available to the data analyst. Let $\de(R)$ denote the descendants of $R$, and let $R = r_0, \de(R) = \de$ denote the covariate values we observed naturally for an individual, while $R=r$ denotes the recourse values. Then, the constraint term in Eq.~\ref{eq:guarantee-term} can be expanded as follows:
\begin{align}
    P(\widehat{Y}_{R=r} = 1 \mid V=v) &= \sum_{\de^*} P(\widehat{Y}_{R=r, \de(R)=de^*} = 1, \de(R)_{R=r}=\de^* \mid R = r_0, \de(R) = \de) \\
    &= \sum_{\de^*} \mathbb{1}(f_{\widehat{y}}(r, \de^*)=1) P(\de(R)_{R=r}=\de^* \mid R = r_0, \de(R) = \de), \label{eq:guarantee-term-expansion}
\end{align}
where $\de^*$ ranges over all possible values of $\de(R)$ that may be obtained after implementing the recourse $R = r$. Note that the expression in Eq.~\ref{eq:guarantee-term-expansion} depends on the joint counterfactual event $\{\de(R) = \de, \de(R)_{R=r}=\de^*\}$. The induced distribution is known to be challenging to evaluate without strong additional assumptions, \citep{tian2000probabilities}, making the problem in Def.~\ref{def:recourse-old} difficult to solve. This is illustrated by the following example:
\begin{example}[Non-Identifiability of Recourse]
    Consider SCMs $\mathcal{M}^{(0)}$, $\mathcal{M}^{(1)}$, defined as follows:
    \begin{align}
        X &\gets U_x \\
        W &\gets X + (-1)^{X \cdot i} U_w \\
        \widehat{Y} &\gets \mathbb{1} ( W \geq 1) \\
        U_x &\sim \text{Bern}(0.5), U_w \sim N(0,1).
    \end{align}
    The two SCMs are equivalent apart from the $W$ mechanism $f_w$. Furthermore, the SCMs generate identical observational and interventional distributions. The causal diagram induced by both SCMs is the chain graph $X \xrightarrow{} W \xrightarrow{} \widehat Y$. 
    
    In the context of recourse, consider an individual with $X(u) = 0, W(u) = 0.5$ corresponding to $(U_x = 0, U_w = 0.5)$. In $\mathcal{M}^{(0)}$ the recourse action $do(X = 1)$ would yield $W_{X = 1} = 1.5$ and $\widehat{Y}_{X=1} = 1$ (i.e., successful recourse), whereas in $\mathcal{M}^{(1)}$ it would yield $W_{X = 1} = 0.5$ and $\widehat{Y}_{X=1} = 0$ (unsuccessful recourse). However, no amount of observational or interventional data would allow these two cases to be distinguished.
\end{example}
The above example provides a negative result for causal recourse. Interestingly, this result is obtained despite the fact that the assignment mechanism of the predictor $\widehat Y$, written $f_{\widehat Y}$ is known to the data analyst\footnote{In this sense, we can see how the setting of algorithmic recourse may differ from the standard literature in causal effect identification, in which none of the mechanisms $\mathcal{F}$ of an SCM are known.}. The remainder of this paper deals with inferring recourse actions while acknowledging this issue. We begin by showing that counterfactual outcomes are not necessarily the appropriate notion for recourse in practice.

\subsection{Counterfactual Reasoning Is Necessary}
The formulation of recourse in Def.~\ref{def:recourse-old} requires reasoning about counterfactual quantities, namely queries of the form
\begin{align} \label{eq:recourse-L3-query}
    P(\widehat{y}_{R = r } \mid V = v). 
\end{align}
Generally, counterfactual queries that belong to the so-called third layer of Pearl's Causal Hierarchy\footnote{We remind the reader that the Pearl's Causal Hierarchy consists of the (1) \textit{observational}, (2) \textit{interventional}, and (3) \textit{counterfactual} layers \citep{bareinboim2020on}.} 
(PCH, for short) are more difficult to infer compared to interventional queries from the second layer of the PCH \citep{bareinboim2020on}. Therefore, it may be tempting to search for a different formulation of causal recourse, based purely on interventional reasoning, such as computing:
\begin{align} \label{eq:recourse-L2-query}
    P(\widehat y \mid do(R = r), V\setminus R = v\setminus r).
\end{align}
However, the situation is more involved, and some type of counterfactual reasoning is needed. As the following example illustrates, the two quantities in Eq.~\ref{eq:recourse-L3-query} and Eq.~\ref{eq:recourse-L2-query} can differ in practice:
\begin{example}[Counterfactual vs. Interventional Reasoning for Recourse] \label{ex:ctf-vs-int}
    Consider the following SCM $\mathcal{M}$:
    % \allowdisplaybreaks
    \begin{align}
        V_1 &\gets U_1 \\
        V_2 &\gets U_2 \label{eq:l2-vs-l3-v2}\\
        V_3 &\gets V_1 \wedge U_3 \label{eq:l2-vs-l3-v3}\\
        V_4 &\gets V_2 \vee U_4 \label{eq:l2-vs-l3-v4} \\
        \widehat{Y} &\gets \mathbb{1}(V_1 + V_2 + V_3 + V_4 \geq 2.5), \\
        \nonumber \\
        U_1 &, U_2, U_3, U_4 \sim \text{Bernoulli}(0.5).
    \end{align}
    Now, consider a person with $(v_1, v_2, v_3, v_4) = (0, 0, 0, 0)$ implementing the recourse recommendation $do(V_1 = 1, V_4 = 1)$. The quantities of interest may be evaluated as:
    \begin{align} \label{eq:rec-ex-L2-query}
        P(\widehat y \mid do(v_1 = 1, v_4 = 1), v_2 = 0, v_3 = 0) &= 0 \\
        P(\widehat y_{v_1 = 1, v_4 = 1} \mid v_1 = 0, v_2 = 0, v_3 = 0, v_4 = 0) &= \frac{1}{2} . \label{eq:rec-ex-L3-query} 
    \end{align}
    For obtaining the quantity in Eq.~\ref{eq:rec-ex-L2-query} we proceed as follows. First, we replace the mechanisms of $V_1, V_4$ with fixed values $V_1 \gets 1, V_4 \gets 1$. 
    Then, in the submodel $\mathcal{M}_{V_1=1,V_4=1}$, we look at the implication of conditioning on $V_2 = 0, V_3 = 0$. First, in $\mathcal{M}_{V_1=1,V_4=1}$ observing that $V_3 = 0$ implies $U_3 = 0$ (since $V_1 = 1$ in this submodel and using Eq.~\ref{eq:l2-vs-l3-v3}). Observing that $V_2 = 0$ implies hat $U_2 = 0$, based on Eq.~\ref{eq:l2-vs-l3-v2}. Thus, we know that $U_3 = U_2 = 0$ and this is sufficient to compute the $\widehat{Y}$ in the submodel $\mathcal{M}_{V_1=1,V_4=1}$, yielding $\widehat{Y} = 0$ with probability 1.

    For the quantity in Eq.~\ref{eq:rec-ex-L3-query} we follow a different procedure. We first update the latent variables according to available evidence. Since we observe $V_1 = 0, V_2 = 0$, it follows that $U_1 = 0, U_2 = 0$. Then, since $V_1 = 0, V_3 = 0$, we cannot infer $U_3$ based on Eq.~\ref{eq:l2-vs-l3-v3}, but conclude that $U_3$ equals $1$ with probability $1/2$. Finally, using mechanism in Eq.~\ref{eq:l2-vs-l3-v4} and evidence $V_2 = 0, V_4 = 0$, we can conclude that $U_4 = 0$. This gives us the probability distribution of $U_1, U_2, U_3, U_4$ given the evidence $V_1 = V_2 = V_3 = V_4 = 0$. To compute the quantity in Eq.~\ref{eq:rec-ex-L3-query}, we replace the mechanisms of $V_1, V_4$ with $V_1 \gets 1, V_4 \gets 1$ and compute $\widehat{Y}$ based on the updated distribution of the $U$ variables. This yields the $\frac{1}{2}$ result in Eq.~\ref{eq:rec-ex-L3-query}, due to the fact that $P(U_3 \mid v_1 = 0, v_2 = 0, v_3 = 0) = \frac{1}{2}$. 
\end{example}
\noindent The above example illustrates how counterfactual reasoning differs from interventional reasoning, and may yield very different results in practice. In fact, queries of the form in Eq.~\ref{eq:rec-ex-L2-query} will generally not allow one to reason about effects of recourse actions. Queries as in Eq.~\ref{eq:rec-ex-L3-query} are closer to what is needed for recourse inference, but still do not capture the problem complexity fully (as discussed in Ex.~\ref{ex:intro}). This motivates the developments in the next section.
\section{New Model of Causal Recourse} \label{sec:new-recourse}
In this section, we discuss why the potential outcome $\widehat{Y}_{R=r}(u)$ may not be the target quantity of interest for solving recourse problems. We start by examining the implicit causal assumptions needed for recourse inference, first illustrated by an example:
\begin{example}[Exam Repetition continued -- Unsuccessful Recourse]
    Consider the exam repetition setting from Ex.~\ref{ex:intro}, and consider a student who did not attend office hours ($X = x_0$), obtained a test score $T = t$, and failed the exam ($Y = 0$). Subsequently, the student was given a recourse recommendation to attend office hours (action $do(X =x_1)$). However, after two weeks, the student comes for the exam repetition without having attended office hours, i.e., $X = x_0$ still (recourse was unsuccessful). However, could it happen that the student obtains a different test score $T=t'$ with $t' \neq t$ on the second attempt, and obtains a passing grade? 
\end{example}
The example illustrates that repeated instantiations of a variable over time, even in the absence of recourse, may be subject to uncertainty. While the assignment mechanisms $\mathcal{F}$ of the SCM (in this case $f_T$) are commonly assumed to be stable across time, the latent $U_t=u_t$ is unlikely to remain identical. In other words, there may be other sources of variation within the latent $u_t$ that are different pre- and post-recourse. In Ex.~\ref{ex:intro}, we discussed the obtained amount of sleep and the stress level in the past two weeks as possible latent variables that may change between different exam attempts. Denote the post-recourse test score by $T_{X=x_0}(u_w^*)$ (given that recourse was unsuccessful, and $X = x_0$ after recourse). In fact, since the student's ability is also affecting the test score, there may be some relation between the latent $U_t = u_t$ pre-recourse and the latent $U_t^* = u_t^*$ post-recourse. That is, $U_t^* \notci U_t$ since some parts of the latent variable are invariant across time. Therefore, based on this observation, a slightly different approach may be necessary to take into account this variability between the latent variables pre- and post-recourse. Instead of assuming that the unit $U = u$ is fixed, we may assume that
\begin{align}
    u = (u^{(f)}, u^{(v)}),
\end{align}
where the first part $u^{(f)}$ represents the fixed, immutable (or intrinsic) latent information on the unit, such as the individual's biology, genetics, or ability. The second part $u^{(v)}$ represents circumstantial information, such as the amount of sleep or stress in the past week. Crucially, in modeling, one needs to account for the fact that this second, circumstantial part $u^{(v)}_{t}$ may be resampled after the implementation of the recourse action. 

\subsection{Causal Recourse Building Blocks}
To represent these dynamics in a more fine-grained way, we introduce a new building block called recourse SCM:
\begin{definition}[Recourse SCM] \label{def:recourse-scm}
    Let $\mathcal{M} = \langle \mathcal{F}, P(u)\rangle$ be an SCM over variables $V, U$. For $R \subset V$ let $R = r$ denote the value of a recourse intervention. A recourse SCM $\recscm$ is a tuple
    \begin{align}
        \langle \mathcal{F}, P(u), R=r, P(u^* \mid u) \rangle,
    \end{align}
    where $P(u^* \mid u)$ is the distribution over the units after the recourse action $do(R = r)$. A recourse SCM is said to be Markovian if $\langle \mathcal{F}, P(u) \rangle$ is Markovian and
    \begin{align}
        P(u^* \mid u) = \prod_{i=1}^{n} P(u^*_i \mid u_i).
    \end{align}
\end{definition}
Based on the definition of a recourse SCM, we can also define notions of a recourse distribution and recourse sampling, analogous to the definitions of observational and counterfactual distributions (Defs.~\ref{def:obs-dist} and \ref{def:ctf-dist}):
\begin{definition}[Recourse Distribution]
    Let $\recscm$ be a recourse SCM, and let $V, V^*$ denote variables pre- and post-recourse, respectively. The recourse distribution is then given by
    \begin{align}
        P_{\recscm}(v, v^*) = \sum_u \mathbb{1}(V(u)=v, V^*(u^*)=v^*)P(u) P(u^* \mid u). 
    \end{align}
\end{definition}
The definition is now illustrated through an example:
\begin{example}[Exam Repetition continued] \label{ex:recourse-scm}
    Consider the exam repetition setting from Ex.~\ref{ex:intro}, and suppose it is described by the following SCM $\mathcal{M}$:
    \begin{align}
        X &\gets U_X \\
        T &\gets X + 2U_1 - U_2 + U_3 \label{eq:f-T-mech}\\
        Y &\gets \mathbb{1} (T \geq 1) \\
        \nonumber \\
        U_X &\sim \text{Bernoulli(0.5)},\; U_1, U_2, U_3 \sim N(0, 1),
    \end{align}
    corresponding to the causal diagram in Fig.~\ref{fig:intro-ex}, with $U_1, U_2, U_3$ corresponding to latent ability, stress, and sleep levels. Office hours attendance ($X$) is a Bernoulli$(0.5)$ random variable, meaning that half of the students attend office hours. The $f_T$ mechanism says that students who attend office hours perform better, with the performance influenced positively by ability ($U_1$) and sleep ($U_3$), and influenced negatively by stress levels ($U_2$). For an individual who did not attend office hours, $X(u) = x_0$, obtained a test score $T(u) = \frac{1}{2}$, and failed the exam $Y(u) = 0$, we consider an intervention $do(X=1)$ (attending office hours), and two alternative distributions over the post-recourse noise variables $U^*$ \\
    
\begin{minipage}{.4\linewidth}
\vspace{-0.15in}
\begin{empheq}[left={P^{(a)}(u^* \mid u) = \empheqlbrace}]{align} % { = }{align*}
  U^*_1 &= U_1  \label{eq:exam-rep-pu-star-1}\\
  U^*_2 &\sim N(0, 1) \\
  U^*_3 &\sim N(0, 1),\label{eq:exam-rep-pu-star-3}
\end{empheq}
\end{minipage}%
\hfill \text{ vs. } \hfill
\begin{minipage}{.45\linewidth}
\vspace{-0.175in}
\begin{empheq}[left={P^{(b)}(u^* \mid u) = \empheqlbrace}]{align}
  U^*_1 &= U_1  \\
  U^*_2 &\sim N(1, 1) \\
  U^*_3 &\sim N(-1, 1).
\end{empheq} 
\end{minipage}

\vspace{0.1in}
\noindent The distribution $P(u^* \mid u)$ determines how the noise variables are sampled post-recourse, possibly conditional on the initial noise variables $U = u$. In particular, for both $P^{(a)}, P^{(b)}$ assume that the latent ability of the student $(U_1)$ remains invariant post-recourse, written $U_1^* = U_1$. However, the two distributions differ according to the distribution over the $U^*_2, U^*_3$ variables. In $P^{(a)}$, post-recourse stress and sleep levels $U_2^*, U_3^*$ follow a $N(0, 1)$ distribution, the same distribution as $U_2, U_3$ variables pre-recourse. In $P^{(b)}$, however, post-recourse stress and sleep levels follow different distributions, $N(1, 1)$ and $N(-1, 1)$, meaning that post-recourse students have higher degrees of stress and lower levels of sleep. 
\end{example}
We now discuss different possible ways to think about recourse inference. In the oracle case, if the true recourse SCM $\recscm$ were available to us, we would be able to compute the recourse outcome $\widehat{Y}_{R=r}(u^*)$ exactly. For each variable under recourse, labeled $R_i^*$, we have $R^*_i \gets r_i$ (due to recourse), whereas for each variable $V_i^*$ not under recourse, one evaluates
\begin{align} \label{eq:rec-mechanism}
    V_i^* (u_i^*) \gets f_{V_i} (\Pa_i^*, u_i^*).
\end{align}
Once all the post-recourse variables $V^*$ are computed (made possible by the exclusive knowledge of the latent recourse variables $U_i^*$), we can simply evaluate the recourse outcome using the classifier $\widehat{Y}$. This approach describes how the ground truth outcome $\widehat{Y}_{R=r}(u^*)$ (which we are interested in) may be recovered from the recourse SCM. We place this outcome in the first column of Fig.~\ref{fig:recourse-inference-targets}, which provides a mental map of possible recourse inference targets.
Inferring outcomes deterministically is rarely possible, however, and we thus need to resort to probabilistic reasoning. 
A slightly weaker inference target than obtaining $\widehat{Y}_{R=r}(u^*)$ would be to condition on the unit $U=u$, that is, the pre-recourse latent variables. Since $U=u$ (pre-recourse), and $U^* = u^*$ (post-recourse) share information, conditioning on $U=u$ would provide useful information about the random outcome $\widehat{Y}_{R=r}(U^*)$. This approach (corresponding to the second column in Fig.~\ref{fig:recourse-inference-targets}) would still require the knowledge of the SCM (since we need to know the values of the latent $U=u$), which is almost never available to us. Therefore, we need to use a slightly weaker probabilistic target, and instead of conditioning on $U=u$, we make use of the available information on each individual (given in terms of the observed vector of covariates $V = v$) and conditioning on it instead (column three in Fig.~\ref{fig:recourse-inference-targets}). 
\begin{figure}
    \centering
    \scalebox{0.95}{
    \begin{tikzpicture}

% Define styles
\tikzstyle{every node}=[font=\normalsize, draw=black, minimum width=2.5cm]
\tikzstyle{label}=[font=\scriptsize,align=center]
\tikzset{labnode/.style={font=\normalsize, align=center, draw = black, fill = gray!20, rotate=90, minimum width=1.8cm}}

% Horizontal line with points and labels
\newcommand{\xstr}{1.2}
\draw[thick] (1.5 * \xstr,0) -- (12 * \xstr,0);
\draw[fill] (1.5 * \xstr,0) circle [radius=0.1] node[below, yshift=-0.3cm, align=center,draw=black] {$u^*$-specific \\ (ground truth)};
\draw[-] (4.5 * \xstr,0.3) - ++(0,-0.6) node[below,align=center,draw=black] {$u$-specific\\ (oracle)};
\draw[-] (7.5 * \xstr,0.3) -- ++(0,-0.6) node[below,draw=black,yshift=-0.25cm] {$v$-specific};
\draw[-] (10.5 * \xstr,0.3) -- ++(0,-0.6) node[below,draw=black,yshift=-0.25cm] {no conditioning};

% Labels above the points
\node at (1.5 * \xstr, 1.2) {\(\widehat{Y}_{R=r}(u^*)\)};
\node at (4.5 * \xstr, 1.2) {\(P(\widehat{Y}_{R=r}=1 \mid U = u)\)};
\node at (7.5 * \xstr, 1.2) {\(P(\widehat{Y}_{R=r}=1 \mid V = v)\)};
\node at (10.5 * \xstr, 1.2) {\(P(\widehat{Y}_{R=r}=1)\)};

% Labels below the line
\node[labnode] at (-1, 1.2) {Inferential \\ Target};
\node[labnode] at (-1, -2) {Markov \\ factors};

% Markov factors below the points
\node at (1.5 * \xstr, -2) {\(V_i^* \leftarrow f_{V_i}(\Pa_i^*, u_i^*)\)};
\node at (4.5 * \xstr, -2) {\(V_i^* \mid \Pa_i^*, U_i\)};
\node at (7.5 * \xstr, -2) {\(V_i^* \mid \Pa_i^*, \Pa_i, V_i\)};
\node at (10.5 * \xstr, -2) {\(V_i^* \mid \Pa_i^*\)};

\end{tikzpicture}
    }
    \caption{Visualization of different possible inference goals for recourse.}
    \label{fig:recourse-inference-targets}
\end{figure}
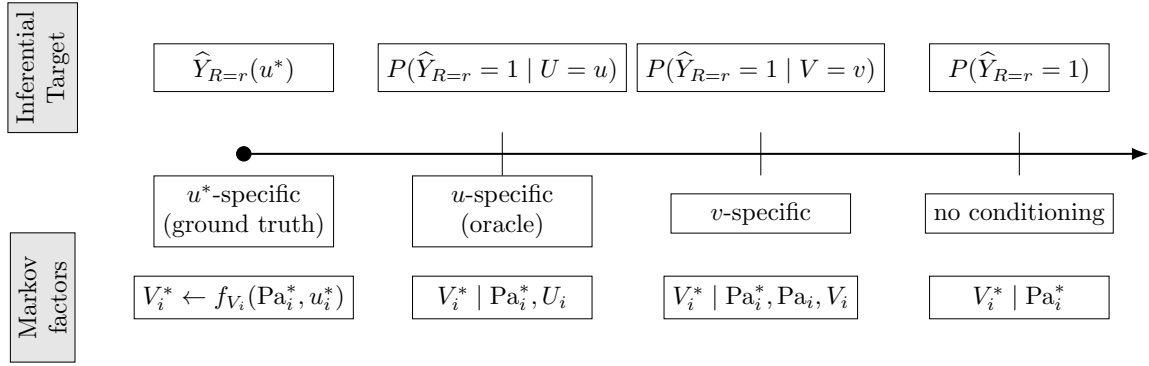
Finally, a baseline approach we mention is to drop the information on the individual pre-recourse, and simply estimate $P(\widehat{Y}_{R=r}=1)$ without conditioning on any information (Fig.~\ref{fig:recourse-inference-targets}, column four). In this way, however, we may drop useful information. In order perform recourse inference, in this manuscript we opt for conditioning on the observed covariate pre-recourse $V =v$, yielding the following problem definition:
\begin{definition}[Causal Recourse Under Resampling] \label{def:new-recourse}
    Let $\mathcal{M}^{P(u^*)}_{R=r}$ be a recourse SCM, and let $\text{cost}(R=r, V = v)$ be the cost of implementing a recourse action $R = r$ for an individual with attributes $V=v$. The problem of finding the optimal recourse action is given by:
    \begin{align}
        \argmin_{R, r} &\quad  \text{cost(R = r, V = v)} \label{eq:new-cost-term}\\
        \text{subject to}&\quad P_{\mathcal{M}^{P(u^*)}_{R=r}}(\widehat{Y}_{R=r} = 1 \mid V=v) \geq t(r),
        \label{eq:new-guarantee-term}
    \end{align}
    where $t(r)$ is a threshold that may vary according to the recourse action $R=r$, and the probability $P_{\mathcal{M}^{P(u^*)}_{R=r}}$ is computed
    based on the recourse SCM $\mathcal{M}^{P(u^*)}_{R=r}$, or data generated from $\mathcal{M}^{P(u^*)}_{R=r}$.
\end{definition}
The emphasis in the above definition is that (i) there may be a type of \textit{distribution change} in the sampling of units $P(u^* \mid u)$ compared to the initial distribution $P(u)$; and (ii) the evaluation of the probability of success in Eq.~\ref{eq:new-guarantee-term} should be considered with respect to the joint distribution over $\mathcal{M}^{P(u^*)}_{R=r}$ and not just the initial $\mathcal{M}$. In particular, the observed $V=v$ are sampled from $\mathcal{M}$ but evaluating $\widehat{Y}_{R=r}(u^*)$ post-recourse requires information beyond just $\mathcal{M}$ since it depends on the post-recourse distribution over the units $P(u^* \mid u)$; and finally (iii) the probability of success in Eq.~\ref{eq:new-guarantee-term} is conditioned on $V = v$, that is, all the observed information. Ideally, in the oracle case, we would want to compute $\widehat{Y}_{R=r}(u^*)$  (or the slightly weaker counterpart  $P_{\mathcal{M}^{P(u^*)}_{R=r}}(\widehat{Y}_{R=r} = 1 \mid U=u)$). The latter approach, even though arguably more informative, would amount to conditioning on information that is never observed, and thus does not provide us with a practicable formulation of recourse. Therefore, conditioning on $V = v$ is a relaxation that still leverages all the available information on the specific individual while defining a possibly feasible problem. We next illustrate the challenges in evaluating the optimization problem in Def.~\ref{def:new-recourse} through our running example:
\begin{example}[Exam Repetition continued]
    Consider the recourse SCM $\mathcal{M}^{P(u^*)}_{R=r}$ in Ex.~\ref{ex:recourse-scm} with a post-recourse distribution $P^{(a)}$ in Eq.~\ref{eq:exam-rep-pu-star-1}-\ref{eq:exam-rep-pu-star-3}. Consider an individual with the latent variables values $(u_X, u_1, u_2, u_3) = (0, 1, 1, -1)$. The individual attains the value $T(u) = 0$ and thus fails the exam on the first attempt. We are now interested in computing the probability of success after a recourse action $do(X=1)$. Based on $\mathcal{M}^{P(u^*)}_{R=r}$, we can see that
    \begin{align}
        T_{X=1}(U^*) = 1 + 2U_1^* - U_2^* + U_3^*.
    \end{align}
    Based on $P^{(a)}$, conditioning on $u_1 = 1$ would imply that $u_1^* = 1$, and $U_2^*, U_3^* \sim N(0, 1)$, i.e.,
    \begin{align} \label{eq:u-specific-pu}
         T_{X=1}(U^*) \mid U=u \; \overset{d}{=} \; 1 + 1 - U_2^* + U_3^* \sim N(2, 2),
    \end{align}
    which implies $P_{\mathcal{M}^{P(u^*)}_{X=1}}(\widehat{Y}_{R=r} = 1 \mid U=u) =P(T_{X=1}(U^*) > 1 \mid U=u) = 1-\Phi(-\frac{\sqrt 2}{2}) \approx 0.76$, where $\Phi$ is the cumulative distribution of $N(0, 1)$.
    
    However, as discussed above, conditioning on $U = u$ is not feasible since this information is almost never available to the data analyst. Thus, we condition on the observed values $X = 0, T = 0$. Based on the SCM, we can infer that
    \begin{align} \label{eq:gauss-slice}
        X(u) = 0, \; T(u) = 0 \implies 2U_1 - U_2 + U_3 = 0.
    \end{align}
    Let $U_T = (U_1, U_2, U_3)$ and $U^*_T = (U_1^*, U_2^*, U_3^*)$. Let the vector $a = (2, -1, 1)$ correspond to the coefficients of $U_1, U_2, U_3$ in the $f_T$ mechanism in Eq.~\ref{eq:f-T-mech}. Based on $P(u), P(u^* \mid u)$, we know that
    \begin{align}
        \begin{pmatrix}
            a^TU_T \\ a^TU^*_T
        \end{pmatrix} \sim N \left ( 
        \begin{pmatrix}
            0 \\ 0
        \end{pmatrix},
        \begin{pmatrix}
            6 & 4 \\
            4 & 6
        \end{pmatrix}
        \right ),
    \end{align}
    since $2U_1 - U_2 + U_3 \sim N(0, 6)$ and Cov$(2U_1 - U_2 + U_3, 2U^*_1 - U^*_2 + U^*_3) = 4$. Therefore, by properties of bivariate Gaussian variables, Eq.~\ref{eq:gauss-slice} implies:
    \begin{align} \label{eq:v-specific-pu}
        2U^*_1 - U^*_2 + U^*_3 \mid 2U_1 - U_2 + U_3 = 0 \sim N(0, \frac{10}{3}).
    \end{align}
    Thus, $P_{\mathcal{M}^{P(u^*)}_{X=1}}(\widehat{Y}_{R=r} = 1 \mid V=v) = 1 - \Phi(0) = \frac{1}{2}$ based on Eq.~\ref{eq:v-specific-pu}. The example highlights the distinction between a unit-level recourse probability (conditional on the unobserved $U = u$, Eq.~\ref{eq:u-specific-pu}) vs. a covariate-level recourse probability (conditional on the observed $V=v$, Eq.~\ref{eq:v-specific-pu}).
\end{example}
In the sequel, we discuss different approaches for inference of recourse probabilities in practice.

\subsection{Markovian Factorization of the Recourse Distribution} \label{sec:markov-recourse}
In this section, we discuss how the recourse problem in Eqs.~\ref{eq:new-cost-term}-\ref{eq:new-guarantee-term} may be broken down. In this paper, we work with Markovian models, meaning that the latent noise variables $U_i$ are independent (i.e., there is no hidden confounding). Clearly, solving the recourse optimization problem requires the evaluation of
\begin{align} \label{eq:post-rec-yhat}
    P_{\mathcal{M}^{P(u^*)}_{R=r}}(\widehat{Y}_{R=r} = 1 \mid V=v)
\end{align}
appearing in the condition in Eq.~\ref{eq:new-guarantee-term}. Here, $\widehat{Y}$ is a classifier, which takes the post-recourse variables $V^*$ as an input. Therefore, evaluating the expression in Eq.~\ref{eq:post-rec-yhat} requires some way of inferring the joint distribution over $P(v, v^*)$. In Fig.~\ref{fig:recourse-network}, we provide a graphical representation based on which the inference approach described in this manuscript can be understood. In Fig.~\ref{fig:recourse-network}, the observed variables $V_1,\dots, V_n$ appear, together with the recourse variables $V_1^*, \dots, V_n^*$. For each $V_i$, there is a variable part of the latent $U_i$, labeled $U_1^{(v)}$, and a fixed part of $U_i$, labeled $U_i^{(f)}$. The key idea is that the fixed part $U_i^{(f)}$ is assumed to be the same pre- and post-recourse, while the variable part $U_1^{(v)}$ may change. Therefore, the corresponding recourse variable $V_i^*$ is influenced by $U_1^{(f)}$, and also $U_1^{(v)*}$, with the latter being possibly different than $U_1^{(v)}$. Importantly, based on Fig.~\ref{fig:recourse-network}, we can also see that
\begin{align}
    V_i^* \ci V\setminus \{V_i, \Pa_i\} \mid V_i, \Pa_i, \Pa_i^*.
\end{align}
In words, once we condition on $V_i$ and its set of pre- and post-recourse parents $\Pa_i, \Pa_i^*$, the post-recourse $V_i^*$ is independent of all other pre-recourse variables $V\setminus \{V_i, \Pa_i\}$. This means that the conditional distribution $P(v^* \mid v)$ can be factorized as:
\begin{align} \label{eq:recourse-markov-fact}
    P(v^* \mid v) = \prod_{V_i \notin R} P(v_i^* \mid pa_i^*, v_i, pa_i) \times \prod_{V_i \in R} \mathbbm{1}(V_i = r_{V_i}),
\end{align}
where $r_{V_i}$ is the value of $V_i$ in the vector of recourse values $r$. We refer to the terms $P(v_i^* \mid pa_i^*, v_i, pa_i)$ as Markov factors, corresponding to the distribution of $V_i^*$ in the joint distribution $P(v^* \mid v)$ (see bottom row of Fig.~\ref{fig:recourse-inference-targets}). 
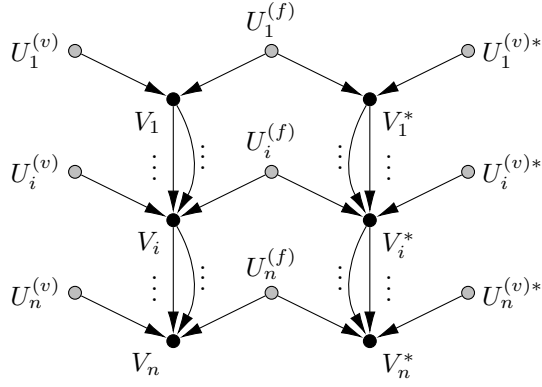
\begin{figure}[t]
    \centering
    \scalebox{1}{
    \begin{tikzpicture}[SCM,scale = 1]
        \pgfsetarrows{latex-latex};
            \newcommand{\xshift}{1.3}
            \newcommand{\yshift}{1.6}
            \newcommand{\mygray}{gray!50}

            % observed variables
            
            \node(v1) at (0 * \xshift, 1 * \yshift)[label={below left}:{$V_1$}, point];
            \node(vi) at (0 * \xshift, 0 * \yshift)[label={below left}:{$V_i$}, point];
            \node(vn) at (0 * \xshift, -1 * \yshift)[label={below left}:{$V_n$}, point];

            \node(u1) at (-1 * \xshift, 1.4 * \yshift)[label=left:{$U^{(v)}_1$}, point, fill=\mygray];
            \node(ui) at (-1 * \xshift, 0.4 * \yshift)[label=left:{$U^{(v)}_i$}, point, fill=\mygray];
            \node(un) at (-1 * \xshift, -0.6 * \yshift)[label=left:{$U^{(v)}_n$}, point, fill=\mygray];
            
            % recourse variables
            \node(v1s) at (2 * \xshift, 1 * \yshift)[label={below right}:{$V_1^*$}, point];
            \node(vis) at (2 * \xshift, 0 * \yshift)[label={below right}:{$V_i^*$}, point];
            \node(vns) at (2 * \xshift, -1 * \yshift)[label={below right}:{$V_n^*$}, point];

            \node(u1s) at (3 * \xshift, 1.4 * \yshift)[label=right:{$U^{(v)*}_1$}, point, fill=\mygray];
            \node(uis) at (3 * \xshift, 0.4 * \yshift)[label=right:{$U^{(v)*}_i$}, point, fill=\mygray];
            \node(uns) at (3 * \xshift, -0.6 * \yshift)[label=right:{$U^{(v)*}_n$}, point, fill=\mygray];

            \node(u1f) at (1 * \xshift, 1.4 * \yshift)[label=above:{$U^{(f)}_1$}, point, fill=\mygray];
            \node(uif) at (1 * \xshift, 0.4 * \yshift)[label=above:{$U^{(f)}_i$}, point, fill=\mygray];
            \node(unf) at (1 * \xshift, -0.6 * \yshift)[label=above:{$U^{(f)}_n$}, point, fill=\mygray];

            % arrows
            \path (u1) edge (v1);
            \path (ui) edge (vi);
            \path (un) edge (vn);

            \path (u1f) edge (v1);
            \path (uif) edge (vi);
            \path (unf) edge (vn);

            \path (u1s) edge (v1s);
            \path (uis) edge (vis);
            \path (uns) edge (vns);

            \path (u1f) edge (v1s);
            \path (uif) edge (vis);
            \path (unf) edge (vns);

            \path (v1) edge (vi);
            \path (vi) edge (vn);
            
            \path (v1s) edge (vis);
            \path (vis) edge (vns);

            % dots
            %\node(v1tovi) at ($(v1)!0.4!(vi) + (0.25, 0)$) {$\vdots$};
            \node at ($(v1)!0.5!(vi) + (-0.25, 0)$) {$\vdots$};
            \node at ($(vi)!0.5!(vn) + (-0.25, 0)$) {$\vdots$};

            \node at ($(v1s)!0.5!(vis) + (0.25, 0)$) {$\vdots$};
            \node at ($(vis)!0.5!(vns) + (0.25, 0)$) {$\vdots$};

            % connect the dots
            \path (v1) edge[bend left = 30] node[pos=0.4, right] {$\vdots$} (vi);
            \path (vi) edge[bend left = 30] node[pos=0.4, right] {$\vdots$} (vn);

            \path (v1s) edge[bend right = 30] node[pos=0.4, left] {$\vdots$} (vis);
            \path (vis) edge[bend right = 30] node[pos=0.4, left] {$\vdots$} (vns);
    \end{tikzpicture}
    }
    \caption{Informal causal diagram over observational and recourse variables. Vertical dots indicate arbitrary arrows and variables between the observed ($V$) and recourse variables ($V^*$).}
    \label{fig:recourse-network}
\end{figure}
We remark that similar reasoning as above could be used if we considered conditioning on $U_i=u_i$ instead of the pair $(V_i, \Pa_i)$. Note that
\begin{align}
    V_i^* \ci U \setminus U_i \mid U_i, \Pa_i^*,
\end{align}
which then implies the factorization
\begin{align}
    P(v^* \mid u) = \prod_{V_i \notin R} P(v_i^* \mid pa_i^*, u_i) \times \prod_{V_i \in R} \mathbbm{1}(V_i = r_{V_i}). 
\end{align}
However, in this paper, we are generally interested in conditioning on $V=v$, although we discuss some alternative approaches later on.

Various inferential challenges arise when considering the factorization in Eq.~\ref{eq:recourse-markov-fact}. Firstly, one may ask whether the conditional distribution $P(v_i^* \mid pa_i^*, v_i, pa_i)$ can be recovered in the case when there is no post-recourse data, i.e., $V^*$ is not observed. Secondly, one may ask how to infer this conditional distribution if some (limited) amount of recourse data is available.
The remainder of the section deals with these inferential challenges, and is organized as follows (see Fig.~\ref{fig:recourse-overview} for a schematic overview). 
First, in Sec.~\ref{sec:PRS} we formally introduce post-recourse stability (PRS), which tells us whether the distribution over the latent variables pre- and post-recourse is the same (intuitively, in Fig.~\ref{fig:recourse-network}, whether $U_i^{(v)}$ and $U_i^{(v)*}$ are distributionally the same). Assuming we have access to observational data (pre-recourse), the ability to infer the effects of recourse actions depends on PRS. In Sec.~\ref{sec:msc-obs}, we show how to perform recourse inference from observational data under PRS, based on a copula model. In Sec.~\ref{sec:msc-rec}, we introduce the concept of \textit{recourse data} -- assuming that data is available on individuals after they have performed recourse actions (that is, we have both pre-recourse samples $V$ and post-recourse samples $V^*$ for the same set of individuals). Then, we discuss how to perform recourse inference from combinations of observational and recourse data under PRS. Finally, in Sec.~\ref{sec:mic-rec}, we demonstrate that recourse inference is possible from combinations of observational and recourse data even if the PRS does not hold.

\subsection{Post-Recourse Stability (PRS)} \label{sec:PRS}
We begin the discussion by stating a set of conditions under which reasoning about recourse will be possible from observational data alone:
\begin{definition}[Post-Recourse Stability] \label{def:post-recourse-stability}
    Let $\mathcal{M} = \langle\mathcal{F}, P(u)\rangle$ be the SCM pre-recourse, and $\mathcal{M}^* = \langle\mathcal{F}_{R=r}, P^*(u)\rangle$ post-recourse. Suppose the unit $u = (u^{(f)}, u^{(v)})$ is partitioned into fixed and variable exogenous features, that is, only $u^{(v)}$ is resampled after recourse. Then, post-recourse stability (PRS) is defined as:
    \begin{enumerate}[label=(\alph*)]
        %\item $\mathcal{F}^*_{-R} = \mathcal{F}_{-R}$,
        \item the exogenous variables of non-descendant variables ($\nde(R)$) of the recourse action $R=r$ remain unchanged 
        \begin{align} \label{eq:u-nde-invariant}
            U^*_{\nde(R)} = U_{\nde(R)},
        \end{align}
        \item for all $V_i$ descendants of $R$, we have that
        \begin{align}
            %\forall\; V_i \in \de(R) 
            U^{*(f)}_{i} &= U^{(f)}_{i}, \text{ and } \label{eq:u-1-invariant}\\ 
            P(U^{*(v)}_{i} = u^{(v)}_{i} \mid U^{*(f)}_{i} = u^{(f)}_{i}) &= P(U^{(v)}_{i} = u^{(v)}_{i} \mid U^{(f)}_{i} = u^{(f)}_{i}). \label{eq:u-2-stable}
        \end{align}
        % or written together as
        % \begin{align}
        %     P(U^*_i = u^*_i \mid U_i = u_i) = \mathbbm{1}(u^{*(f)}_i = u^{(f)}_i) P(U^{(v)}_{i} = u^{*(v)}_i \mid U^{(f)}_i = u_i^{(f)}). 
        % \end{align}
        %\item $P(u_2 \mid u_1) = P^*(u_2 \mid u_1)$.
    \end{enumerate}
    % \begin{align}
    %     \text{(a) } \mathcal{F}^*_{-R} = \mathcal{F}_{-R}, \quad
    %     \text{(b) } P(u^{(2)} \mid u^{(1)}) = P^*(u^{(2)} \mid u^{(1)}).
    % \end{align}
\end{definition}
Despite the non-trivial notation, the two conditions are quite intuitive. 
%Firstly, we assume that the causal mechanisms $\mathcal{F}_{-R}$ of the non-intervened variables remain unchanged. 
Firstly, Eq.~\ref{eq:u-nde-invariant} states that for all non-descendants of the recourse action $do(R=r)$, the latent variables remain unchanged. This is natural because these variables are not intervened upon, and neither are any of their ancestors (hence, we expect them to remain invariant).
Secondly, Eq.~\ref{eq:u-1-invariant} states that the fixed part of the latent $u_i$, labeled $u_i^{(f)}$, remains invariant, while Eq.~\ref{eq:u-2-stable} states that the sampling of the variable (coincidental) part $u^{(v)}$ of the unit $u$ is the same as in the original SCM $\mathcal{M}$, given the fixed part of the unit $u^{(f)}$. 
This condition precludes certain scenarios, for example, settings in which applicants perform better on repeated tries of a standardized test; in this case, the $u^{(v)}$ distribution post-recourse may change, even though the $u^{(f)}$ (intrinsic ability) would not. 
\begin{example}[Exam Repetition -- Post-Recourse Stability]
    Going back to the recourse SCM from Ex.~\ref{ex:recourse-scm}, consider the two distributions over post-recourse latent variables, $P^{(a)}$ and $P^{(b)}$. For both distributions, the latent $U_1$ remains unchanged pre- and post-recourse. For $P^{(a)}$, the post-recourse latent variables $U^*_2, U^*_3$ are distributed as $N(0, 1)$, same as pre-recourse. However, this is not true for $P^{(b)}$. The distribution of $U^*_2$ is $N(1, 1)$ post-recourse, whereas it was $N(0, 1)$ pre-recourse. 
\end{example}
As shown next, settings where PRS holds are interesting since reasoning about recourse from observational data is feasible (the theorem's proof is given in Appendix~\ref{appendix:theorem-proofs}):
\begin{theorem}[Post-Recourse Stability $\implies$ Margin Stability] \label{thm:prs-implies-msc}
    Let $U \sim P(u)$ with $P(u)$ the distribution over units before recourse, and $U^* \sim P^*(u)$ with $P^*(u)$ the distribution after recourse. Post-recourse stability (Def.~\ref{def:post-recourse-stability}) implies margin stability, written as
    \begin{align}
        f_{V_i}(\pa_i, U_i) \overset{d}{=} f_{V_i}(\pa_i, U_i^*) \quad \forall\; i, \pa_i \text{ fixed}. \label{eq:margin-stability}
    \end{align}
\end{theorem}
The importance of this result stems from the following. Suppose there is a unit $u$ with $V(u) = v$ implementing a recourse action $R = r$. Let $V^*_i, \,\Pa^*_i$ denote a variable and its parents post-recourse, respectively. Based on Eq.~\ref{eq:margin-stability} we know that, given the parents under the action $do(R = r)$, denoted by $\Pa^*_i(u^*) = \pa^*_i$, the distribution of $V^*_i$ after recourse is the same as the observational distribution of $V_i$ given $\Pa_i = \pa^*_i$, i.e.,  
\begin{align}
    V^*_i \mid \Pa^*_i = \pa^*_i \overset{d}{=} V_i \mid \Pa_i = \pa^*_i. \label{eq:stability-imp}
\end{align}
The l.h.s.\;of Eq.~\ref{eq:stability-imp} is the marginal target distribution after recourse, and the r.h.s.\;is available from observational data. Thus, we know the marginal distribution of $V^*_i$ conditional on parents $\Pa^*_i$.
Nonetheless, note that we have still not leveraged the fact that the values of $V_i,\, \Pa_i$ before recourse were observed, which may further improve our inference.
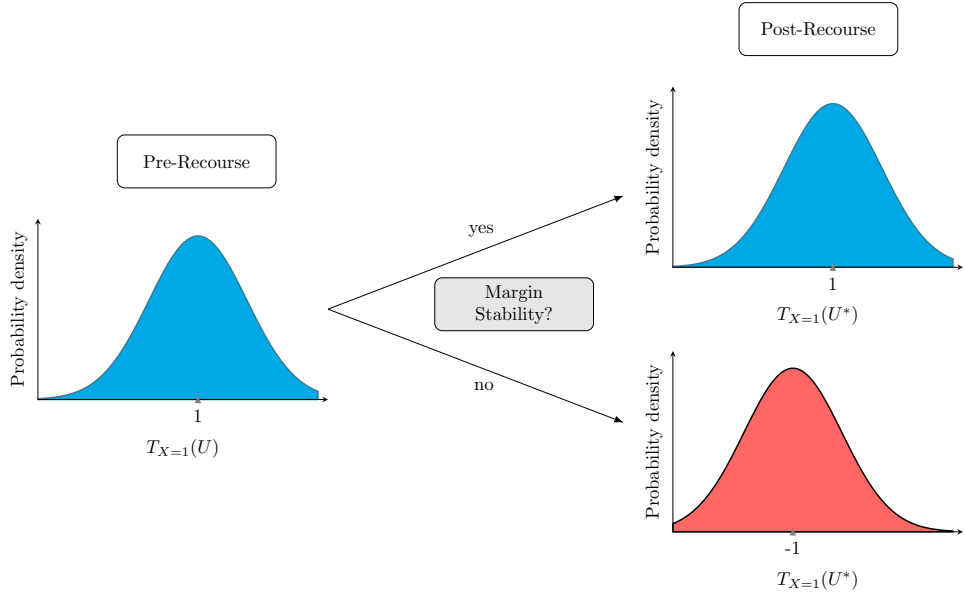
\begin{figure}
    \centering
    \scalebox{0.7}{
    \begin{tikzpicture}[
    %node distance=2cm,
    box/.style={rectangle, draw, minimum width=3cm, minimum height=1cm, align=center, rounded corners},
    title/.style={rectangle, draw, fill=blue!30, minimum width=4cm, minimum height=1cm, align=center, rounded corners},
    decision/.style={rectangle, draw, fill=gray!20, minimum width=3cm, minimum height=1cm, align=center, rounded corners}
]
    % Define the Gaussian function
    \newcommand\gaussian[2]{exp(-(#1)^2/(2*#2)) / (sqrt(2*pi*#2))}

    % Pre-Recourse box
    \node[box] at (0, 4.5) {Pre-Recourse};

    % Draw the Gaussian plots
    \begin{axis}[
        name=pre-rec,
        at={(-3cm,0)},
        no markers,
        domain=-7:7,
        samples=100,
        xmin=-7,xmax=7.5,
        ymin=0,ymax=0.18,
        axis lines*=left,
        xlabel={$T_{X=1}(U)$},
        ylabel={Probability density},
        xtick=\empty,
        ytick=\empty,
        extra x ticks={1},
        extra x tick labels={1},
        extra x tick style={grid style={draw=none}},
        enlargelimits=false,
        clip=false,
        axis on top,
        grid = major,
        y post scale=0.6,
        x post scale=0.8,
        axis lines = left
    ]
    \addplot [thick,cyan!50!black, fill=cyan!95!black, -] {\gaussian{x-1}{6}} \closedcycle;
    \end{axis}

    \node[box] at (11.75, 7) {Post-Recourse};

    \node[decision] at (6, 1.8) {Margin\\Stability?};

    % Draw the Gaussian plots
    \begin{axis}[
        name=post-msc,
        at={(9cm,2.5cm)},
        no markers,
        domain=-7:7,
        samples=100,
        xmin=-7,xmax=7.5,
        ymin=0,ymax=0.18,
        axis lines*=left,
        xlabel={$T_{X=1}(U^*)$},
        ylabel={Probability density},
        xtick=\empty,
        ytick=\empty,
        extra x ticks={1},
        extra x tick labels={1},
        extra x tick style={grid style={draw=none}},
        enlargelimits=false,
        clip=false,
        axis on top,
        grid = major,
        y post scale=0.6,
        x post scale=0.8,
        axis lines = left
    ]
    \addplot [thick,cyan!50!black, fill=cyan!95!black, -] {\gaussian{x-1}{6}} \closedcycle;
    \end{axis}

    \begin{axis}[
        name=post-msc-false,
        at={(9cm,-2.5cm)},
        no markers,
        domain=-7:7,
        samples=100,
        xmin=-7,xmax=7.5,
        ymin=0,ymax=0.18,
        axis lines*=left,
        xlabel={$T_{X=1}(U^*)$},
        ylabel={Probability density},
        xtick=\empty,
        extra x ticks={-1},
        extra x tick labels={-1},
        extra x tick style={grid style={draw=none}},
        ytick=\empty,
        enlargelimits=false,
        clip=false,
        axis on top,
        grid = major,
        y post scale=0.6,
        x post scale=0.8,
        axis lines = left
    ]
    \addplot [thick, fill=red!60!white, domain=-7:7, -] {\gaussian{x+1}{6}} \closedcycle;
    %\addplot [thick,cyan!50!black, fill=red!95!black, -] {\gaussian{x+1}{6}};
    \end{axis}

    \draw[->, -Latex, shorten >=1cm] (pre-rec.east) -- (post-msc.west) node[midway, above left] {yes};
    \draw[->, -Latex, shorten >=1cm] (pre-rec.east) -- (post-msc-false.west)node[midway, below left] {no};
\end{tikzpicture}
    }
    \caption{Schematic representation of margin stability conditions (Thm.~\ref{thm:prs-implies-msc}) related to Ex.~\ref{ex:msc-implications}. On the left we have the marginal distribution pre-recourse, whereas on the right we have the possible post-recourse distributions. If MSC hold, then the post-recourse distribution will be equal to the pre-recourse one (in blue). However, if MSC do not hold, then the post-recourse distribution may differ (in red).}
    \label{fig:msc-schema}
\end{figure}
\begin{example}[Exam Repetition -- MSC Implication] \label{ex:msc-implications}
    Consider the recourse SCM from Ex.~\ref{ex:recourse-scm} and the two post-recourse distributions $P^{(a)}, P^{(b)}$. Suppose now we are interested in the distribution of test scores after recourse, assuming all units of the population implemented the recourse action $do(X = 1)$. Then, using the $f_T$ mechanism in Eq.~\ref{eq:f-T-mech}, we have that for $P^{(a)}$ 
    \begin{align}
        T_{X=1}(U_T^*) \sim N(1, 6)
    \end{align}
    since $X = 1$ and $2U^*_1 - U^*_2 + U^*_3 \sim N(0, 6)$ according to $P^{(a)}$. Crucially, note that the pre-recourse distribution of test scores equals
    \begin{align}
        T_{X=1}(U_T) \sim N(1, 6),
    \end{align}
    that is, the marginal distributions pre- and post-recourse are equal. This is a key implication of post-recourse stability.

    On the other hand, for $P^{(b)}$, we have that $2U^*_1 - U^*_2 + U^*_3 \sim N(-2, 6)$, and therefore
    \begin{align}
        T_{X=1}(U_T^*) \sim N(-1, 6)
    \end{align}
    In this case, the post-recourse marginal differs from the pre-recourse marginal. A schematic representation is shown in Fig.~\ref{fig:msc-schema}.
\end{example}
The key observation is that, in absence of data on individuals implementing recourse actions, and without PRS, the marginal distribution after recourse is not recoverable. Therefore, in absence of PRS and recourse data, inference on recourse actions would generally not be feasible. Under PRS, at least the marginal distribution following from recourse action can be identified. We will leverage this insight in the sequel. 

\paragraph{Connection to Unobservability of Potential Outcomes.} Before continuing, we draw an important connection to the inherent unobservability of joint potential outcomes $\widehat{Y}(u), \widehat{Y}_{R=r}(u)$. Previously, we discussed how a unit $u$ can be decomposed into its fixed part $u^{(f)}$ and variable part $u^{(v)}$ that is resampled. For cases in which the proportion of the variance explained by $u^{(v)}$ tends to $0$, that is, the unit's value is almost entirely fixed, we expect that the counterfactual outcome $\widehat{Y}_{R=r}(u)$ is equal to the post-recourse outcome $\widehat{Y}_{R=r}(u^*)$ obtained by the unit $u$ after implementing recourse. Therefore, in such a setting, if the post-recourse sample $\widehat{Y}_{R=r}(u^*)$ is available, it would allow us to effectively observe the counterfactual $\widehat{Y}_{R=r}(u)$. The opposite extreme would be when the proportion of variance explained by $u^{(f)}$ tends to $0$, in which case the unit is considered as entirely resampled, i.e., $U \ci U^*$ in the previous notation.

\subsection{Non-Parametric Inference Under Post-Recourse Stability -- Observational Data} \label{sec:msc-obs}
In Sec.~\ref{sec:markov-recourse} we discussed several different ways of incorporating the pre-recourse information of the individual when computing recourse probabilities (recall Fig.~\ref{fig:recourse-inference-targets}). As argued, when inferring the distribution of the post-recourse variable $V_i^*$, in the oracle case, we would consider the post-recourse latent variables $U_i^*$, or their pre-recourse counterparts $U_i$. In practice, however, we would have to resort to conditioning on the pair $(V_i, \Pa_i)$ according to the factorization in Eq.~\ref{eq:recourse-markov-fact}. 
However, there is another route to inferring recourse probabilities. Importantly, we know that the latent $U_i$ shares information with the $U_i^*$, written $U_i \notci U_i^*$. In this section, we discuss a relaxed approach for doing so, in which instead of conditioning on $U_i$ itself, which is unobserved, we condition on an important proxy of $U_i$, namely the quantile of $V_i = v_i$ in the distribution $P(V_i \mid \Pa_i = \pa_i)$:
\begin{definition}[Distribution Quantile]
    Let $f_{V_i}$ be a mechanism of the SCM, and let $U_i$ be the latent variable associated with $V_i$. We define the quantile of $V_i=v_i$ given parents $\pa_i$ as the
    \begin{align}
        Q(v_i \mid \pa_i) = P(V_i \leq v_i \mid \pa_i).
    \end{align}
\end{definition}
The quantile $q_i$ results from a many-to-one mapping of latents $u_i \mapsto q_i$. 
Nonetheless, the quantile $q_i$ is a function of and contains information on the latent $u_i$, i.e., $q_i = q_i(u_i)$. Similarly for the post-recourse quantile, $q^*_i = q^*_i(u^*_i)$. 
Since in general $U_i \notci U^*_i$, it also implies that $Q_i \notci Q^*_i$, meaning that the quantiles pre- and post-recourse share information. 
To leverage this fact for inference, we introduce a model for the coupling of pre-recourse quantiles $q_i$ and post-recourse quantiles $q_i^*$.

\paragraph{Applying Sklar's Theorem.} To model the coupling of quantiles we apply Sklar's theorem \citep{sklar1959fonctions}. The theorem states that a bivariate cumulative distribution function $H(x_1, x_2)$ can be expressed in terms of its marginals $F_1, F_2$ and a copula $C$, i.e.,
\begin{align}
    H(x_1, x_2) = C(F_1(x_1), F_2(x_2)).
\end{align}
The copula $C$ represents the coupling of the quantiles of the marginal distributions $F_1, F_2$ that defines the joint distribution and is a commonly used modeling tool in statistics, finance, econometrics, and a number of other domains \citep{jaworski2010copula}. By an application of Sklar's theorem, for any pair of distributions
\begin{align}
    V_i \mid \Pa_i = \pa_i \text{ and } V^*_i \mid \Pa^*_i = \pa^*_i,
\end{align}
there exists a copula that relates their quantiles. 

\paragraph{Why Quantiles?} In the above construction, quantiles play a key role as a connection point with the latent variables $U$. Quantiles can be thought of as the best possible non-parametric proxy for the unobserved unit $u$, since the quantiles are (i) independent of the value of the parents $\Pa_i = \pa_i$; (ii) independent of the mechanism $f_{V_i}$ of $V_i$. However, the many-to-one mapping of a unit $Q_i: u_i \mapsto q_i$ can be understood as a type of quotient operation, with all units $\{u_i: Q_i(u_i) = q_i \}$ considered as congruent. In other words, no amount of data will allow an analyst to distinguish among the units in the set $\{u_i: Q_i(u_i) = q_i \}$ regardless of the level of sophistication of the inference method. However, units corresponding to different values of $q_i, q'_i$ may be distinguished, and this fact may allow us to better infer recourse probabilities.

\begin{example}[Exam Repetition -- Congruence of Units]
    Consider the recourse SCM in Ex.~\ref{ex:recourse-scm} with a post-recourse distribution $P^{(a)}$ and an individual with the latent variables values $(u_X, u_1, u_2, u_3) = (0, 1, 1, -1)$. For this individual, we have that
    \begin{align}
        X(u) = 0, T(u) = 0. 
    \end{align}
    The marginal distribution of $T \mid X = 0$ is $N(0, 6)$. Therefore, the unit $(0, 1, 1, -1)$ corresponds to the 50\% quantile (i.e., the median). Importantly, any other unit satisfying
    \begin{align} \label{eq:exam-congruence}
        2u_1 - u_2 + u_3 = 0
    \end{align}
    also corresponds to the 50\% quantile. Thus, the 50\% quantile corresponds to the subspace given by $\{(u_1, u_2, u_3) \in \mathbbm{R}^3: 2u_1 - u_2 + u_3 = 0\}$.
\end{example}
The above example illustrates that quantiles are mixing different kinds of units. However, these units may share some similarity. In the sequel, the Kendall's $\tau$ parameter of the copula will represent a measure of how similar the units within a fixed quantile are.

\paragraph{Inference Based on Frank's Copula.} In this paper, we use Frank's copula \citep{frank1979simultaneous} parameterized by the Kendall's $\tau$ correlation coefficient \citep{kendall1948rank} for the coupling of quantiles of $V_i \mid \Pa_i = \pa_i$ and $V^*_i \mid \Pa^*_i = \pa^*_i$. The joint distributions over the quantiles for values of $\tau \in \{ -0.5, 0, 0.5, 1\}$ are shown in Fig.~\ref{fig:frank-tau}.
\begin{figure}
    \centering
    \includegraphics[width=1\textwidth]{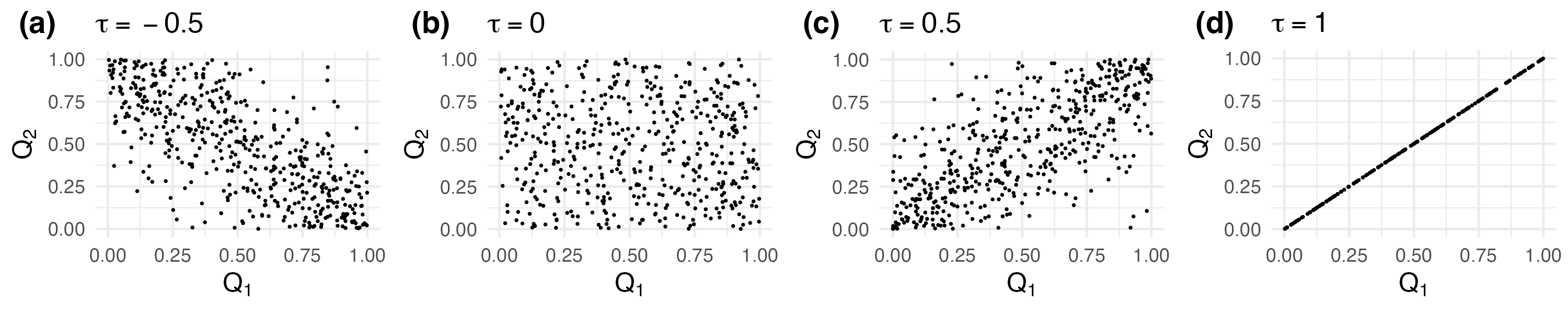}
    \caption{Frank's copula coupling of quantiles for different values of $\tau \in \{-0.5, 0, 0.5, 1\}$.}
    \label{fig:frank-tau}
\end{figure}
As can be seen from the figure, the larger the $\tau$, the more closely related the quantiles are. For $\tau = 1$, the quantiles pre- and post-recourse are identical. Values of $\tau \in (0, 1)$ still have the property that the expectation of $\ex [Q^*_i \mid Q_i = q_i] = q_i$, but the spread around the expected value increases as $\tau$ approaches $0$. For $\tau = 0$, the quantiles are independent, representing the so-called independence copula. In this case, given the value of $\Pa_i = \pa_i$ pre-recourse, knowing the value of $V_i = v_i$ does not tell us anything additional about the post-recourse value $V_i^* = v_i^*$ (in other words, only the pre-recourse parent set $\Pa_i$ possibly contains information about $V_i^*$). For $\tau < 0$, the quantile coupling is reversed, so that $\ex [Q^*_i \mid Q_i = q_i] = 1 - q_i$.

In Alg.~\ref{algo:frank-obs} we describe the procedure for estimating the effect of a recourse action under post-recourse stability based on observational data. The $\tau$ parameter of the copula is an input to the algorithm. We remark that at each node $V_i$, the $\tau$ parameter may depend on the parents $\Pa_i = \pa_i$, i.e., $\tau = \tau(\pa_i)$. However, for simplicity, we assume $\tau$ is constant for each node $V_i$ and each value of the parents $\Pa_i = \pa_i$. The $\tau$ parameter is an input since it cannot be inferred purely from observational data. Therefore, Alg.~\ref{algo:frank-obs} can be run with varying levels of $\tau$, and we can inspect how this affects the post-recourse distribution through a sensitivity-type analysis with $\tau$ acting as the sensitivity parameter (demonstrated empirically in Sec.~\ref{sec:experiments}). 

\begin{algorithm}[t]
    \caption{Copula Inference from Observational Data}
     \begin{algorithmic}[1]
        \Statex \textbullet~\textbf{Inputs:} Causal Diagram $\mathcal{G}$, Observational Data $\mathcal{D}$, Kendall's $\tau$ Parameter of the Frank copula, Recourse Action $R = r$, Individual $V = v$, Number of Monte Carlo samples $N$.
        \For{$V_i \in V$ in topological order}
            \State perform quantile regression $V_i \sim \Pa_i$ to learn the quantile function $Q(v_i \mid \pa_i)$ using $\mathcal{D}$
        \EndFor
        \For{$\de(R) \setminus R \in V$ in topological order}
            \State infer the quantile $q_i$ of $v_i \mid \pa_i$ from $Q(v_i \mid \pa_i)$, labeled $q_i$
            \State draw $N$ samples from the conditional Frank's copula $C_{\tau}(\cdot \mid Q_1 = q_i)$ labeled $\{q_i^{(k)}\}_{k=1:n}$ 
            \State infer the $N$ values of $v^{(k)*}_i$ as
            \begin{align}
                v_i^{(k)*} \gets Q^{-1}(q_i^{(k)} \mid \pa_i^{(k)*})
            \end{align}
            where parents $\pa^{(k)*}$ are either known or possibly obtained in the previous steps.
        \EndFor
        \Statex \textbullet~\textbf{Output:} $N$ Monte Carlo values $\{v^{(k)*}\}_{k=1:n}$ under recourse $R = r$ for individual $V=v$.
     \end{algorithmic}
     \label{algo:frank-obs}
\end{algorithm}
The algorithm first iterates through all nodes in the graph and learns the conditional distribution of $V_i \mid \Pa_i$. Then, for each descendant node of the recourse action $R = r$ that is not subject to intervention itself, we first infer the quantile $q_i$ of the observation $V_i = v_i$ for the individual. Then, based on the Frank's copula $C_\tau$, we sample the $N$ Monte Carlo quantiles conditional on $q_i$ and under the recourse action, labeled $\{q_i^{(k)}\}_{k=1:n}$. We then obtain the Monte Carlo values for $v^{(k)*}_i$ based on the quantiles $q_i^{(k)}$ and the quantile function $Q(\cdot \mid \pa^{(k)*})$ of the distribution $V_i \mid \pa^{(k)*}_i$. The main drawback of the procedure is the inherent lack of identifiability of $\tau$ from observational data. However, if we have the resulting data from implementing recourse decisions in practice, $\tau$ may in fact be inferred. Moreover, we can also test whether the Frank copula model holds for our data, as discussed in the sequel.

\subsection{Non-Parametric Inference Under Post-Recourse Stability -- Recourse Data} \label{sec:msc-rec}
To describe the problem of learning from recourse data, we first provide a formal definition of recourse data:
\begin{definition}[Recourse Data] \label{def:recourse-data}
    Let $\mathcal{M}^{P(u^*)}_{R=r} = \langle \mathcal{F}, P(u), R=r, P(u^* \mid u) \rangle$ be a recourse SCM. Samples are drawn from $\mathcal{M}^{P(u^*)}_{R=r}$ as follows:
    \begin{enumerate}[label=(S\arabic*)]
        \item Draw $U = u$ according to $P(u)$, \label{step:sample-u}
        \item Evaluate $V(u) = v$ according to $\mathcal{F}$, \label{step:eval-F}
        \item If $\widehat{Y}(u) = 0$, then \label{step:recourse-sample}
            \subitem (S3a) Draw $U^* = u^*$ according to $P(u^* \mid u)$,
            \subitem (S3b) Evaluate $V^*(u^*)$ based on mechanisms $\mathcal{F}_{R \gets r}$.
    \end{enumerate}
    Let $P_{R=r}^*(V)$ be the described post-recourse distribution, and let $V_{R=r}^*$ be the random variable describing the post-recourse covariates.
\end{definition}
The definition of recourse data captures several key characteristics of the setting. 
Firstly, in \ref{step:sample-u}-\ref{step:eval-F}, the pre-recourse observational sample is obtained. After this, in \ref{step:recourse-sample}, a post-recourse sample is drawn based on $P(u^* \mid u)$ and $\mathcal{F}_{R\gets r}$, but this happens only for \textit{individuals with $\widehat{Y}(u) = 0$}. In Def.~\ref{def:recourse-data}, we suppose that a recourse sample is drawn for all individuals with $\widehat Y (u) = 0$, while in practice only a subset of these individuals may implement recourse actions (suppose $S = 1$ is an indicator of whether recourse is implemented). Cases where the probability of implementing recourse, written $P(S = 1 \mid U = u)$ depends only on the observed variables, 
\begin{align}
    P(S = 1 \mid U = u) = P(S = 1 \mid V(u))
\end{align}
can also be handled using the methods described in this paper, assuming that we have
\begin{align}
    \delta < P(S = 1 \mid V(u)) < 1 - \delta
\end{align}
for all $u$ with  $\widehat Y(u) = 0$ (in this case, a simple reweighing of the recourse data would be, in infinite samples, equivalent to obtaining recourse data on every individual with $\widehat Y (u) = 0$). More complex cases, where the probability of implementing recourse depends on $U$ or even $U^*$ are not considered, and are left for future work.

While the full recourse SCM $\mathcal{M}^{P(u^*)}_{R=r}$ would, in principle, allow us to generate post-recourse samples for any unit $u$, in practice, no post-recourse samples will be available for units with $\widehat{Y}(u) = 1$ since those with a positive outcome will not implement recourse actions. Therefore, any recourse data collected in practice will always be conditional on $\widehat Y = 0$, representing a kind of selection bias \citep{hernan2004structural,bareinboim2012controlling, bareinboim2022recovering}. 

\begin{example}[Exam Repetitions -- Recourse Data]
Consider the recourse SCM from Ex.~\ref{ex:recourse-scm} and the distribution over post-recourse latent variables $P^{(a)}$. Then, consider all units with $X(u) = 0$. Based on the $f_T$ mechanism in Eq.~\ref{eq:f-T-mech}, only the units with $2u_1 - u_2 + u_3 < 1$ will have $T(u) < 1$ and would, therefore, possibly be interested in implementing recourse. For those with $X(u) = 1$, only units with $2u_1 - u_2 + u_3 < 0$ would have $T(u) < 1$ and thus may implement recourse. In conclusion, recourse data would be available for
\begin{align}
    u \in \mathcal{U} \text{ s.t. } \{ u_X = 0,  2u_1 - u_2 + u_3 < 1 \} \vee \{ u_X = 1,  2u_1 - u_2 + u_3 < 0 \}
\end{align}
\end{example}
Based on the definition of recourse data, we can now define the concept of recourse learning:
\begin{algorithm}[t]
    \caption{Copula Inference from Observational and Recourse Data}
     \begin{algorithmic}[1]
        \Statex \textbullet~\textbf{Inputs:} Causal Diagram $\mathcal{G}$, Observational Data $\mathcal{D}$, Recourse $do(R = r)$, Recourse Data $\mathcal{D}^*$.
        \For{$V_i \in \{\de(R) \setminus R\}$ in topological order}
            \State learn quantile function of $v_i \mid \pa_i$, labeled $Q(v_i \mid \pa_i)$, using $\mathcal{D}$
            \State infer pre-recourse quantile $\hat{q}^{(k)}_i$ of $v^{(k)}_i \mid \pa_i$ from $Q(v_i \mid \pa_i)$ $\forall k$ appearing in both $\mathcal{D}, \mathcal{D}^*$
            \State infer post-recourse quantile $\hat{q}^{(k)*}_i$ of $v^{(k)*}_i \mid \pa^{(k)*}_i$ from $Q(v^*_i \mid \pa^*_i)$
            \State compute $\hat{\tau}_i$ as maximizer of conditional copula likelihood,
            \begin{align}
                \hat{\tau}_i = \argmax_{\tau} \log f_{\tau}(\{\hat{q}^{(k)*}_i\}^{n_{\text{rec}}}_{k=1} \mid \{\hat{q}^{(k)}_i\}^{n_{\text{rec}}}_{k=1}).
            \end{align}
            \State compute baseline Cramer-von Mises statistic $S_{i,n}$ as in Eq.~\ref{eq:cvm-baseline}
            \For{$m = 1, \ldots, M$} \Comment{Bootstrap loop}
                \State draw $\tilde{q}^{(k,m)*}_i \sim \text{Frank}(\hat\tau_i) \mid \hat{q}^{(k)}_i$ for all $k$ \Comment{Simulate from $H_0$}
                \State set $\tilde{v}^{(k,m)*}_i = Q^{-1}(\tilde{q}^{(k,m)*}_i \mid \pa^{(k)*}_i)$ \Comment{Map quantile to $V_i$ value}
                \State refit $Q^{(m)}$ on a bootstrapped dataset $\mathcal{D}^{(m)}$
                \State re-estimate $\hat{q}^{(k,m)}_i, \hat{q}^{(k,m)*}_i$ from $Q^{(m)}$ applied to $v^{(k)}_i, \tilde{v}^{(k,m)*}_i$, respectively
                \State compute $\hat\tau^{(m)}_i$ from pairs $(\hat{q}^{(k, m)}_i, \hat{q}^{(k,m)*}_i)$
                \State compute $S^{(m)}_{i,n}$ as the Cramer-von Mises statistic for the $m$-th bootstrap sample (Eq.~\ref{eq:cvm-boot})
            \EndFor
            \State \label{line:p-val-def} compute the p-value $p_i = \frac{1}{M} \sum_{m=1}^M \mathbb{1}(S_{i, n} < S_{i, n}^{(m)})$
        \EndFor
        \Statex \textbullet~\textbf{Output:} $\tau$ estimates $\hat{\tau}_i$, p-values $p_i$
     \end{algorithmic}
     \label{algo:frank-recourse}
\end{algorithm}
\begin{definition}[Recourse Learning] \label{def:recourse-learning}
    Consider an SCM $\mathcal{M}$ and its observational distribution $P(V)$. Further, consider a collection of recourse distributions $\{  P_{R=r}^*(V)\}_{R=r \in \mathcal{R}}$ as described in Def.~\ref{def:recourse-data}. The task of learning from recourse data is to recover the conditional distribution
    \begin{align}
        V_{R=r}^* \mid V = v
    \end{align}
    based on samples from $P(V)$ and $\{P_{R'=r'}^*(V)\}_{R'=r' \in \mathcal{R}}$.
\end{definition}
The task of learning from recourse data represents a natural setting in algorithmic recourse. Consider an institution interested in implementing a recourse policy. At first, they may attempt to infer what would happen to individuals under recourse based on observational data from $P(V)$ only. Then, they may start issuing recourse recommendations $\{R = r\} \in \mathcal{R}$, and recording samples of the variables after recourse was implemented, i.e., from $P_{R=r}^*(V)$. These additional samples can help a great deal in inferring the effects of recourse actions, since they allow one to verify whether post-recourse stability holds and whether Frank's copula model holds.    

\paragraph{Selection Bias Based on $\widehat Y = 0$.}  As indicated in the step \ref{step:recourse-sample} of recourse data sampling in Def.~\ref{def:recourse-data}, a post-recourse sample will only be available for units $u$ for whom $\widehat{Y}(u) = 0$
%\footnote{In this manuscript, we for simplicity assume that individuals who choose to perform recourse are selected completely at random -- that is, we do not model the possibility that willingness to perform recourse may be influenced by endogenous variables $V$ or exogenous $U$.}
. Therefore, the recourse data we have access to will always be subject to \textit{selection bias} \citep{hernan2004structural, bareinboim2012controlling, bareinboim2022recovering} based on the outcome $\widehat{Y}$. The key consequence of this is that pre-recourse quantiles for individuals who have post-recourse data \textit{will not follow a $\mathrm{Unif}[0,1]$ distribution}. By definition, if for a variable $V_i = v_i$, we computed its quantile conditional on the parents $\Pa_i = \pa_i$, and we pooled all the obtained quantiles across the population, the resulting distribution of quantiles would be Unif$[0,1]$. However, when we focus on the quantiles of those for whom ultimately $\widehat{Y} = 0$, the resulting distribution of quantiles need not equal Unif$[0,1]$. Care needs to be taken to account for this key feature of recourse data. Copula models are usually intended for distributional coupling where marginal distributions are Unif$[0,1]$, which is not the case for our setting.

\paragraph{Inferring $\tau$ from Recourse Data.} A procedure for inference from observational and recourse data is described in Alg.~\ref{algo:frank-recourse}. This algorithm is intended for cases when a large amount of observational together with a small amount of recourse data is available. First, based on observational data, the quantile function $Q(v_i \mid \pa_i)$ is learned. Then, for individuals with recourse data, pre and post-recourse quantiles for the sample $v_i, v^*_i$ are inferred from the quantile functions $Q(\cdot \mid \pa_i), Q(\cdot \mid \pa^*_i)$. Based on the coupling of quantiles (since pre and post-recourse samples are paired) we can perform conditional maximum likelihood estimation to infer the most likely $\tau$ for the generated data, such that the likelihood $\log f_{\tau}(q^*_i \mid q_i)$ is maximized (see Appendix \ref{appendix:cond-mle} for conditional likelihood expressions). 

\paragraph{Goodness-of-fit for Frank's copula.}
After estimating $\tau$ for each node $V_i$, we verify whether Frank's copula is an appropriate model for the data by performing a goodness-of-fit test. We adapt the approach of \citet{genest2009goodness} to account for estimation error in the quantile estimates $\hat{q}_i, \hat{q}^*_i$, which are obtained by fitting a quantile function $Q(v_i \mid \pa_i)$ (see Alg.~\ref{algo:frank-recourse}). Since the quantiles are estimated rather than computed from known marginals, the estimation error, if unaccounted for, may lead to a spurious rejection of the null hypothesis.
We first compute the empirical cumulative distribution function (ECDF) for the estimated quantile pairs $\hat{q}_i, \hat{q}^*_i$ pre- and post-recourse, defined as
\begin{align}
    C_{i, n} (u, v) = \frac{1}{n_{\text{rec}}} \sum_{k=1}^{n_{\text{rec}}} \mathbbm{1} ( u \leq \hat{q}^{(k)}_i \wedge v \leq \hat{q}^{(k)*}_i).
\end{align}
We estimate $\hat\tau_i$ by maximizing the conditional Frank copula likelihood, and denote by $C_{i, \hat\tau}$ the ECDF of the Frank$(\hat\tau_i)$ copula conditional on $\hat{q}^{(k)}_i$, approximated via Monte Carlo. The baseline Cramer-von Mises statistic is
\begin{align} \label{eq:cvm-baseline}
    S_{i,n} := \sum_{k=1}^{n_{\text{rec}}}  \left( C_{i, n} (\hat{q}^{(k)}_i, \hat{q}^{(k)*}_i) - C_{i, \hat\tau}(\hat{q}^{(k)}_i, \hat{q}^{(k)*}_i) \right)^2.
\end{align}
To construct the null distribution for the Cramer-von Mises statistic, we employ a bootstrap procedure. For each bootstrap repetition $m$, we: (i) draw post-recourse quantiles $\tilde{q}^{(k,m)*}_i$ from the fitted $C_{i, \hat\tau}(\cdot \mid \hat{q}^{(k)}_i)$ copula; (ii) map the quantiles using the inverse of the learned quantile function to obtain bootstrap samples $\tilde{v}^{(k,m)*}_i = Q^{-1}(\tilde{q}^{(k,m)*}_i \mid \pa^{(k)*}_i)$; (iii) refit the quantile function $Q^{(m)}$ on bootstrapped sample of the data $\mathcal{D}^{(m)}$; and (iv) re-estimate the pre- and post-recourse quantiles $\hat{q}^{(k,m)}_i, \hat{q}^{(k,m)*}_i$ using $Q^{(m)}$. The sampling in steps (i) and (iii) ensures that the bootstrap null distribution of the test statistic reflects both the estimation noise in the quantile function and the randomness of the post-recourse quantiles. 
Note that the bootstrap procedure conditions on the pre-recourse quantiles $\hat{q}^{(k)}_i$ throughout, which ensures that selection on $\widehat{Y} = 0$ does not affect inference (i.e., we do not assume uniform margins at any point). 
The Cramer-von Mises statistic for the $m$-th bootstrap sample is given by
\begin{align} \label{eq:cvm-boot}
    S^{(m)}_{i, n} := \sum_{k=1}^{n_{\text{rec}}}  \left( C^{(m)}_{i, n} (\hat{q}^{(k, m)}_i, \hat{q}^{(k,m)*}_i) - C_{i, \hat\tau^{(m)}}(\hat{q}^{(k, m)}_i, \hat{q}^{(k,m)*}_i) \right)^2,
\end{align}
where $\hat\tau^{(m)}_i$ is re-estimated from the bootstrap pairs $(\hat{q}^{(k, m)}_i, \hat{q}^{(k,m)*}_i)$. The p-value is the empirical quantile of $S_{i,n}$ within $\{S^{(m)}_{i,n}\}_{m=1}^M$, computed in Line~\ref{line:p-val-def}. The key question we discuss in the sequel is how to proceed when the null hypothesis of Frank's copula is rejected.

\subsection{Non-Parametric Inference Under Post-Recourse Instability -- Recourse Data} \label{sec:mic-rec}
The goodness-of-fit hypothesis test described in Alg.~\ref{algo:frank-recourse} may be rejected for several reasons\footnote{The test may be rejected if the causal diagram is misspecified and latent confounding exists (i.e., the true model is Semi-Markovian). However, we do not focus on this case, and assume a correctly specified diagram.}:
\begin{enumerate}[label=(\alph*)]
    \item \label{case:diff-copula} The quantile coupling cannot be represented by Frank's copula, 
    \item Post-recourse stability from Def.~\ref{def:post-recourse-stability} does not hold. \label{case:copula-free}
\end{enumerate}
For case \ref{case:diff-copula}, a different copula family may be more appropriate, or a non-parametric model could be introduced (such extensions would be variations of Algs.~\ref{algo:frank-obs} and \ref{algo:frank-recourse}). More interestingly, we now discuss what happens for case \ref{case:copula-free} when Post-recourse stability is violated.

\begin{algorithm}[t]
    \caption{Copula-Free Inference from Observational and Recourse Data}
     \begin{algorithmic}[1]
        \Statex \textbullet~\textbf{Inputs:} Causal Diagram $\mathcal{G}$, Observational Data $\mathcal{D}$, Recourse Action $do(R = r)$, Recourse Data $\mathcal{D}^*$, Number of Monte Carlo samples $N$.
        \For{$\de(R) \setminus R \in V$ in topological order}
            \State \label{line:copula-free-quantreg} regress $V^*_i \sim V_i + \Pa_i + \Pa^*_i$ to learn quantile function $Q(v^*_i \mid v_i, \pa_i, \pa^*_i)$ using $\mathcal{D}, \mathcal{D}^*$
            \State draw $N$ samples from Unif$[0,1]$, labeled $\{q^{(k)}\}_{k=1:n}$, and infer values of $v^{(k)*}_i$ as
            \begin{align}
                v_i^{(k)*} \gets Q^{-1}(q^{(k)} \mid v_i, \pa_i, \pa_i^{(k)*})
            \end{align}
            where values $v_i, \pa_i, \pa_i^{(k)*}$ are either known or obtained in the previous steps.
        \EndFor
        \Statex \textbullet~\textbf{Output:} $N$ Monte Carlo values $\{v^{(k)*}\}^N_{k=1}$ under recourse $R = r$ for individual $V = v$.
     \end{algorithmic}
     \label{algo:copula-free-learning}
\end{algorithm}
The procedure for this case is described in Alg.~\ref{algo:copula-free-learning}. It assumes a large amount of both observational and recourse data. Crucially, the quantile function $Q(v_i \mid \pa_i)$ learned in the observational data is no longer applicable to recourse data (since margin stability does not hold, see discussion around Ex.~\ref{ex:msc-implications}). Nonetheless, we may still be able to learn the correct quantile function based on \textit{recourse data}.
The key question here is whether observational samples could still be useful for inference even if margin stability does not hold. Even though margin stability may not be satisfied, it is still the case that the quantile of $V_i = v_i$ in the observational distribution $V_i \mid \Pa_i = \pa_i$ shares information with the quantile of $V^*_i = v_i^*$ in the interventional, post-recourse distribution $V_i^* \mid \Pa_i^* = \pa_i^*$. To leverage this connection, we learn the quantile of the recourse random variable $V_i^*$ conditional on (i) the observed parents $\pa_i$; (ii) the initial observed value $v_i$; (iii) the recourse parents $\pa^*_i$. The usage of $(v_i, \pa_i)$ in the quantile regression in Line~\ref{line:copula-free-quantreg} serves as conditioning on the quantile $q_i$ of $v_i \mid \pa_i$, since we know this quantile may be correlated with the post-recourse quantile $q_i^*$, and may thus improve inference. 
A theoretical basis for the regression in Line~\ref{line:copula-free-quantreg} is developed in the following theorem, in the case of linear models:
\begin{theorem}[EMSPE Reduction]\label{thm:linear-recourse}
    Let $Y$ denote a variable in a linear Gaussian SCM, and let $X$ denote $\pa(Y)$. Pre-recourse values are denoted by a subscript $0$, while post-recourse values have no subscript. The pre-recourse and post-recourse models can be written as
    \begin{equation}
        Y_0 \gets X_0\beta_0 + \eps_0, \qquad Y \gets X\beta + \eps,
    \end{equation}
    where $(X_{0,i}, X_i)$ are i.i.d.\ jointly Gaussian with $\ex[X_{0,i}X_{0,i}^\top] = \Sigma_0 \succ 0$ and $\ex[X_iX_i^\top] = \Sigma \succ 0$. Further, assume the noise terms $\eps_0, \eps$ are independent of covariates, and have a correlation $\rho$, so that $\eps = \rho\,\eps_0 + \sqrt{1-\rho^2}\,\beps$ with $\eps_0 \sim N(0, \sigma^2 I)$, $\beps \sim N(0, \sigma^2 I)$, and $\eps_0 \ci \beps$. 
    Let $\hbeta_c$ denote the two-stage OLS estimator obtained via:
    \begin{alignat}{2}
        \textup{Stage~I}:&\quad Y_0 &&\overset{\textup{OLS}}{\sim}\; X_0 \quad\text{to obtain } \hbeta_0 \text{ and }\heps_0 = Y_0 - X_0\hbeta_0, \\
        \textup{Stage~II}:&\quad Y &&\overset{\textup{OLS}}{\sim}\; \bX = (X \;\; \heps_0) \quad \text{to obtain} \hbeta_c,
    \end{alignat}
    Next, we label training data with a superscript $a$, and test data with $b$. Let the test prediction risk (conditional on training data) be defined as
    \begin{equation}
        R_{\mathrm{aug}} := \frac{1}{n_b}\,\ex_{\mathrm{test}}\bigl[\|Y^b - \bX^b\hbeta_c\|^2 \;\big|\; \dt\bigr],
    \end{equation}
    where $\dt = \{X_0^a, X^a, \eps_0^a, \beps^a\}$ denotes the training data.
    For the number of training samples $n_a \gg d$, $R_{\mathrm{aug}}$ satisfies
    \begin{equation}
        R_{\mathrm{aug}} \;=\; (1-\rho^2)\,\sigma^2 + \ordp{\sqrt{\frac{d}{n_a}}}.
    \end{equation}
    Since the risk of the standard single-stage $Y \overset{\textup{OLS}}{\sim} X$, written $R_{\mathrm{std}} := \frac{1}{n_b}\,\ex_{\mathrm{test}}\bigl[\|Y^b - X^b\hbeta \|^2 \;\big|\; \dt\bigr]$, equals $\sigma^2 + \ordp{\frac{d}{n_a}}$, we have $R_{\mathrm{aug}} < R_{\mathrm{std}}$ for $\rho \neq 0$ and $n_a$ sufficiently large.
\end{theorem}
The proof is given in Appendix~\ref{appendix:linear-recourse}.
The theorem compares two approaches: (i) the standard approach, in which assuming access to recourse data, we predict a recourse variable based only on its recourse parents; and (ii) the augmented two-stage approach that leverages the fact that pre- and post-recourse noise terms are correlated (e.g., $U \notci U^*$, as discussed in Sec.~\ref{sec:new-recourse}), and computes pre-recourse residuals $\hat\eps$, and adds these residuals into the regression of the post-recourse outcome on its parents, in order to improve inference.
In terms of EMSPE on a test set, the standard approach has an irreducible loss of $\sigma^2$ and an excess loss of $\ordp{\frac{d}{n_a}}$. The two stage approach is different, however: the irreducible loss is reduced by a factor depending on the correlation of noise terms, to a level $(1-\rho^2)\sigma^2$, whereas the excess loss vanishes more slowly compared to the usual case, at the rate $\ordp{\sqrt{\frac{d}{n_a}}}$. Therefore, for large enough samples, the augmented approach outperforms the standard approach, and in this way Thm.~\ref{thm:linear-recourse} provides a basis for Alg.~\ref{algo:copula-free-learning}.
\section{Experiments} \label{sec:experiments}
In this section, we apply the procedures in Alg.~\ref{algo:frank-obs}, \ref{algo:frank-recourse}, and \ref{algo:copula-free-learning} to real and semi-synthetic data. Throughout, we use the Home Equity Line of Credit (HELOC) dataset \citep{heloc2016fico}.
The dataset contains anonymized information on $n=10459$ applicants who applied for a home equity line of credit. It includes features such as risk scores, income levels, loan-to-value ratios, and others that reflect the creditworthiness of the applicants. The goal is determine whether the applicant was deemed high (``Bad") or low (``Good") credit risk.
The experiments are accompanied with vignettes for Algs.~\href{https://dplecko.github.io/causal-recourse/vignettes/helocA.html}{\ref*{algo:frank-obs}}, \href{https://dplecko.github.io/causal-recourse/vignettes/helocB.html}{\ref*{algo:frank-recourse}}, \href{https://dplecko.github.io/causal-recourse/vignettes/helocC.html}{\ref*{algo:copula-free-learning}} available in our \href{https://github.com/dplecko/causal-recourse}{code repository}.

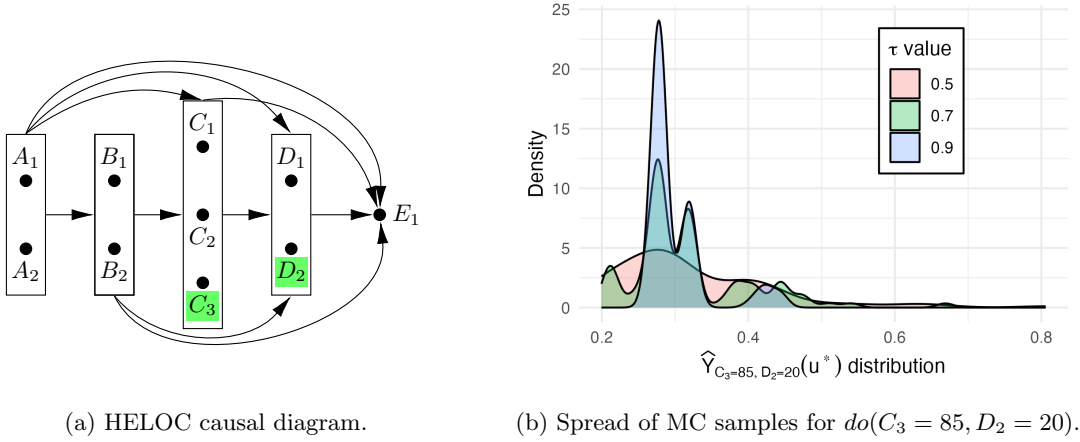
\begin{figure}
    \centering
    \begin{subfigure}{0.45\textwidth}
        \centering
        \begin{tikzpicture}[SCM,scale = 0.9]
        \pgfsetarrows{latex-latex};
            \newcommand{\xshift}{1.3}
            \node(a1) at (0 * \xshift, 0.5)[label=above:{$A_1$}, point];
            \node(a2) at (0 * \xshift, -0.5)[label=below:{$A_2$}, point];
            \node(b1) at (1 * \xshift, 0.5)[label=above:{$B_1$}, point];
            \node(b2) at (1 * \xshift, -0.5)[label=below:{$B_2$}, point];
            \node(c1) at (2 * \xshift, 1)[label=above:{$C_1$}, point];
            \node(c2) at (2 * \xshift, 0)[label=below:{$C_2$}, point];
            \node(c3) at (2 * \xshift, -1)[label=below:{$C_3$}, point];
            \node(d1) at (3 * \xshift, 0.5)[label=above:{$D_1$}, point];
            \node(d2) at (3 * \xshift, -0.5)[label=below:{$D_2$}, point];
            \begin{pgfonlayer}{background}
            \draw [fill=green!60, draw=none] ([shift={(-5pt,-14pt)}]c3.south west) rectangle ([shift={(5pt,-1.5pt)}]c3.south east);
            \draw [fill=green!60, draw=none] ([shift={(-5pt,-14pt)}]d2.south west) rectangle ([shift={(5pt,-1.5pt)}]d2.south east);
            \end{pgfonlayer}
            \node(e1) at (4 * \xshift, 0)[label=right:{$E_1$}, point];

            \node(empty) at (2 * \xshift, -2) {};

            \node[draw,rectangle,fit=(a1) (a2), yscale = 1.5, inner sep=0.5em] (RA) {};
            \node[draw,rectangle,fit=(b1) (b2), yscale = 1.5, inner sep=0.5em] (RB) {};
            \node[draw,rectangle,fit=(c1) (c3), yscale = 1.3, inner sep=0.5em] (RC) {};
            \node[draw,rectangle,fit=(b1) (b2), yscale = 1.5, inner sep=0.5em] (RB) {};
            \node[draw,rectangle,fit=(d1) (d2), yscale = 1.5, inner sep=0.5em] (RD) {};

            \draw[->] (RC.north) to[bend left=40] (e1);

            \draw[->] (RA) to (RB);
            \draw[->] (RB) to (RC);
            \draw[->] (RC) to (RD);
            \draw[->] (RD) to (e1);
            \draw[->] (RA.north) to[bend left = 25] (RC.north);
            \draw[->] (RA.north) to[bend left = 50] (RD.north);
            \draw[->] (RA.north) to[bend left = 80] (e1);
            \draw[->] (RB.south) to[bend right = 50] (RD.south);
            \draw[->] (RB.south) to[bend right = 80] (e1);
        \end{tikzpicture}
        \caption{HELOC causal diagram.}
        \label{fig:heloc-dag}
    \end{subfigure}
    \hfill
    \begin{subfigure}{0.54\textwidth}
        \centering
        \includegraphics[width=0.9\linewidth]{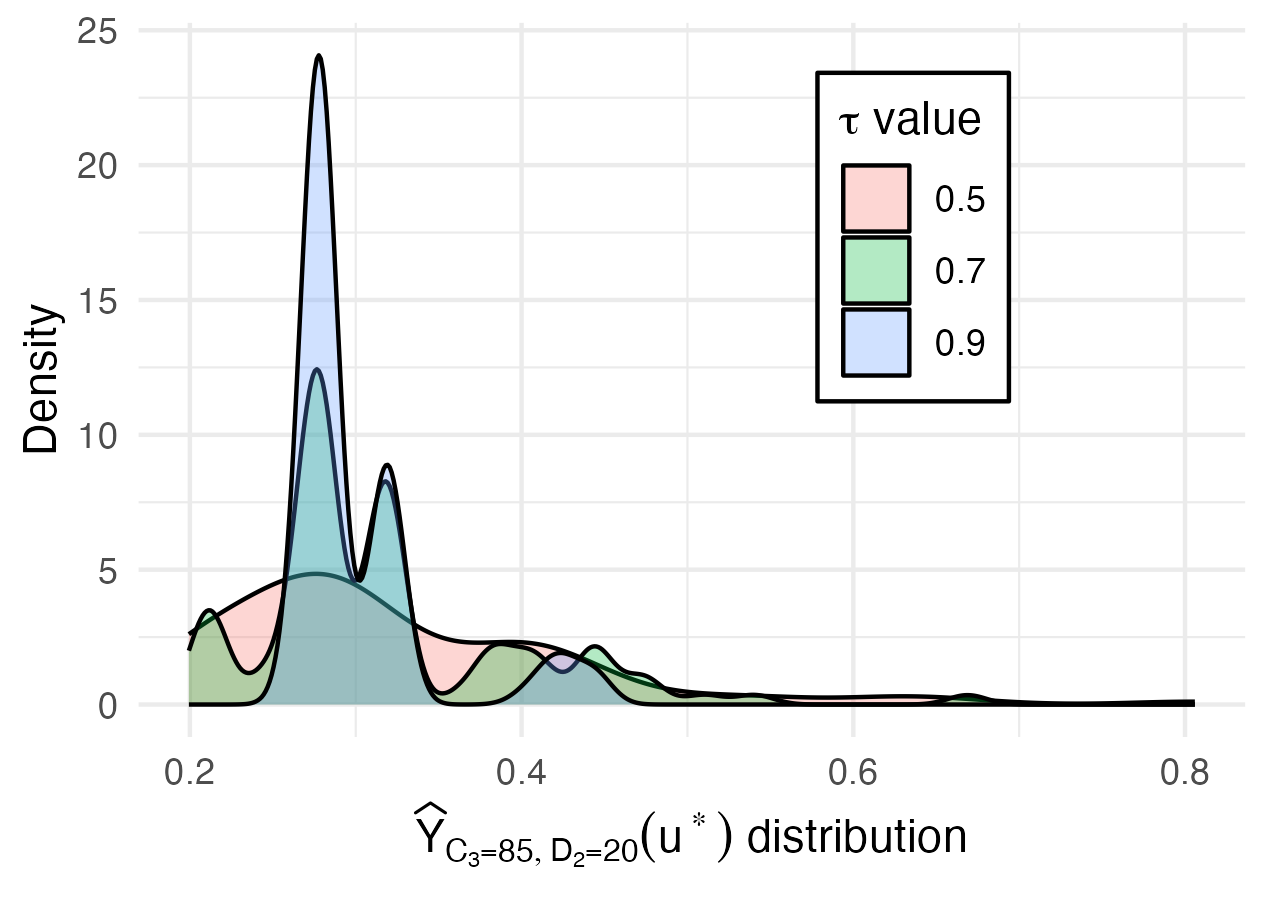}
        \caption{Spread of MC samples for $do(C_3 = 85, D_2 = 20)$.}
        \label{fig:tau-sensitivity}
    \end{subfigure}
    \caption{Causal diagram and spread of $\widehat{Y}(u^*)$ values for different values of $\tau \in \{0.7, 0.8, 0.9, 1\}$. Variables under recourse are marked in green.}
    \label{fig:alg-1}
\end{figure}

% \add{
% % TODO:
% \begin{enumerate}
%     \item Put examples into examples environment,
%     \item Add more detailed description of the HELOC dataset,
%     \item ...
% \end{enumerate}
% }
\subsection{Alg.~\ref{algo:frank-obs} -- A Sensitivity Approach.} We begin by applying Alg.~\ref{algo:frank-obs} to the HELOC dataset. The outcome variable $Y \in \{0, 1\}$ indicates whether the individual's risk for repaying a loan is considered bad or good. First, we fit a random forest \citep{breiman2001random} model to predict the outcome $Y$. We select the top 10 variables according to variable importance based on the Gini index.
We then construct the causal diagram over these 10 variables, by grouping them as follows:
\begin{enumerate}[label=(G\arabic*)]
    \item months since oldest trade ($A_1$), months in file ($A_2$),
    \item number of total trades ($B_1$), number of satisfactory trades ($B_2$),
    \item \% installment trades ($C_1$), \% trades with balance ($C_2$), \% trades never delinquent ($C_3$),
    \item months since last credit report inquiry ($D_1$), revolving balance to credit limit ratio ($D_2$),
    \item external estimate of lending risk ($E_1$).
\end{enumerate}
In particular, we assume that each group G$_i$ points to each downstream group G$_j$ for $j > i$. The causal diagram is shown in Fig.~\ref{fig:heloc-dag}, where each arrow from one cluster to another corresponds to multiple arrows in the fully expanded causal diagram (full diagram shown in Appendix~\ref{appendix:HELOC-diagram}). We then apply Alg.~\ref{algo:frank-obs} and learn the quantile functions $Q(v_i \mid \pa_i)$ using quantile regression forests \citep{meinshausen2006quantile}. We then pick individuals with percent trades never delinquent $C_3 < 70$, and revolving balance ratio $D_2 > 50$, who obtained a negative decision from $\widehat{Y}$. We pick the recourse action $do(C_3 = 85, D_2 = 20)$, and we generate $N = 100$ recourse MC samples for each individual, and with different values of $\tau$. For such individuals, we expect the recourse samples to have higher value of the predictor $\widehat{Y}$, with the spread increasing for $\tau$ values closer to $0$. A prototypical output of such a sensitivity analysis is shown in Fig.~\ref{fig:tau-sensitivity}. As expected, $\tau$ values closer to $0$ correspond to a larger spread, and from the MC samples we can also compute the probability of crossing the decision boundary, i.e., $P_{\tau}(\widehat{Y}_{R=r}(u^*) > \frac{1}{2})$, which in general depends on $\tau$.    

\subsection{Alg.~\ref{algo:frank-recourse} -- Copula Goodness-of-fit.} For applying Alg.~\ref{algo:frank-recourse}, access to recourse data samples is needed. To do so, we create a semi-synthetic HELOC dataset. To simplify the setting slightly, we further reduce the number of variables, by dropping variables $A_1, B_1, B_2,$ and $C_1$. The assumed causal diagram for the semi-synthetic (SeS) HELOC dataset is given in Fig.~\ref{fig:ses-heloc-dag}. Variables $C_2, C_3, D_2, E_1$ that are percentages are scaled into the $[0, 1]$ interval. Then, we fit the maximum likelihood estimator for the following SCM functional form:
\begin{alignat}{2}
    &V_i = A_1:  &&A_1 \gets N(\mu, \sigma^2) \label{eq:ses-heloc-1}\\
    &V_i \in \{C_2, C_3, D_2, E_1\}: \;\; &&V_i \gets \text{Beta}(\alpha(\pa_i), \beta(\pa_i)) \\
    &V_i = D_1:  &&D_1 \gets \text{Geom}(p(\pa_i)) \label{eq:ses-heloc-last}
\end{alignat}
In particular, $\mu, \sigma^2$ are fixed parameters, whereas functions $\alpha(\cdot), \beta(\cdot), p(\cdot)$ are linear functions of the respective arguments. For each variable, we then assume that the mechanisms $\mathcal{F}$ pre- and post-recourse remain the same. However, for the quantile coupling, we distinguish two versions. First, in the SeS-HELOC A we set that quantiles $q_i, q_i^*$ are coupled by Frank's copula with $\tau = \frac{2}{3}$. For the SeS-HELOC B, we set that $Q_i^* \mid Q_i = q_i \sim \text{Unif}[q_i, 1]$, i.e., the post-recourse quantile is strictly larger than the initial quantile. For this dataset, Frank's copula is not valid, and post-recourse stability is not satisfied. In Fig.~\ref{fig:p-values-distr} we show the empirical cumulative distribution function (ECDF) of the p-values for the hypothesis test for datasets A and B over $n_{\mathrm{rep}} = 100$ repetitions of $5 \cdot 10^3$ observational samples and $500$ recourse samples from individuals with a negative decision. For dataset A, where the null hypothesis of Frank's copula is true, the distribution is nearly uniform, confirming that the bootstrap procedure in Alg.~\ref{algo:frank-recourse} correctly accounts for quantile estimation error. For dataset B, the distribution is skewed heavily towards $0$, indicating that we can detect violations of the copula coupling from recourse data. 
% Additionally, in Appendix~\ref{appendix:tau-inference} we show that the true value of the $\tau$ parameter can also be inferred as described in Alg.~\ref{algo:frank-recourse}.

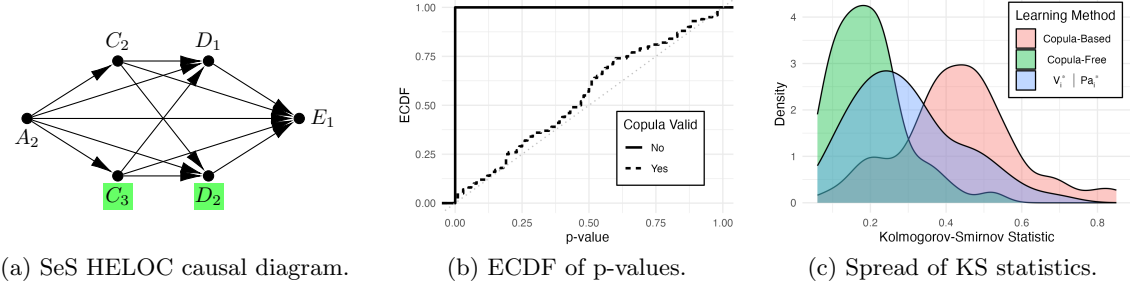
\begin{figure}
    \centering
    \begin{subfigure}{0.33\textwidth}
        \centering
        \begin{tikzpicture}[SCM,scale = 0.8]
        \pgfsetarrows{latex-latex};
            \newcommand{\xshift}{1.5}
            \newcommand{\yshift}{1.9}
            \node(a2) at (1 * \xshift, 0 * \yshift)[label=below:{$A_2$}, point];
            \node(c2) at (2 * \xshift, 0.5 * \yshift)[label=above:{$C_2$}, point];
            \node(c3) at (2 * \xshift, -0.5 * \yshift)[label=below:{$C_3$}, point];
            \node(d1) at (3 * \xshift, 0.5 * \yshift)[label=above:{$D_1$}, point];
            \node(d2) at (3 * \xshift, -0.5 * \yshift)[label=below:{$D_2$}, point];
            \node(e1) at (4 * \xshift, 0 * \yshift)[label=right:{$E_1$}, point];
            \begin{pgfonlayer}{background}
            \draw [fill=green!60, draw=none] ([shift={(-5pt,-14pt)}]c3.south west) rectangle ([shift={(5pt,-1.5pt)}]c3.south east);
            \draw [fill=green!60, draw=none] ([shift={(-5pt,-14pt)}]d2.south west) rectangle ([shift={(5pt,-1.5pt)}]d2.south east);
            \end{pgfonlayer}

            \node(empty) at (\xshift, -1 * \yshift) {};
            
            \path (a2) edge (c2);
            \path (a2) edge (c3);
            \path (a2) edge (d1);
            \path (a2) edge (d2);
            \path (a2) edge (e1);
            \path (c3) edge (d1);
            \path (c3) edge (d2);
            \path (c2) edge (d1);
            \path (c2) edge (d2);
            \path (d1) edge (e1);
            \path (d2) edge (e1);
            \path (c3) edge (e1);
            \path (c2) edge (e1);
        \end{tikzpicture}
        \caption{SeS HELOC causal diagram.}
        \label{fig:ses-heloc-dag}
    \end{subfigure}
    \hfill
    \begin{subfigure}{0.3\textwidth}
        \centering
        \includegraphics[width=\linewidth]{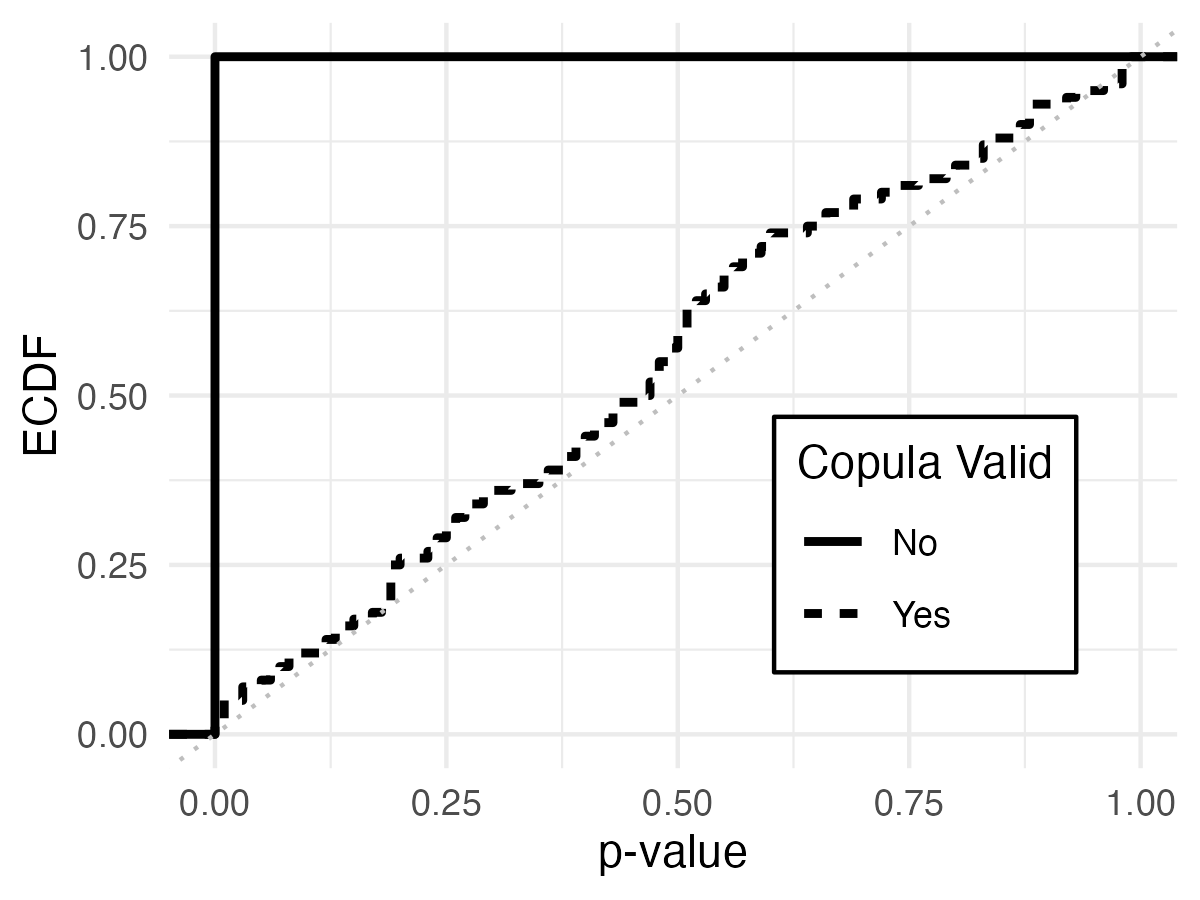}
        \vspace{-0.23in}
        \caption{ECDF of p-values.}
        \label{fig:p-values-distr}
    \end{subfigure}
    \hfill
    \begin{subfigure}{0.32\textwidth}
        \centering
        \includegraphics[width=\linewidth]{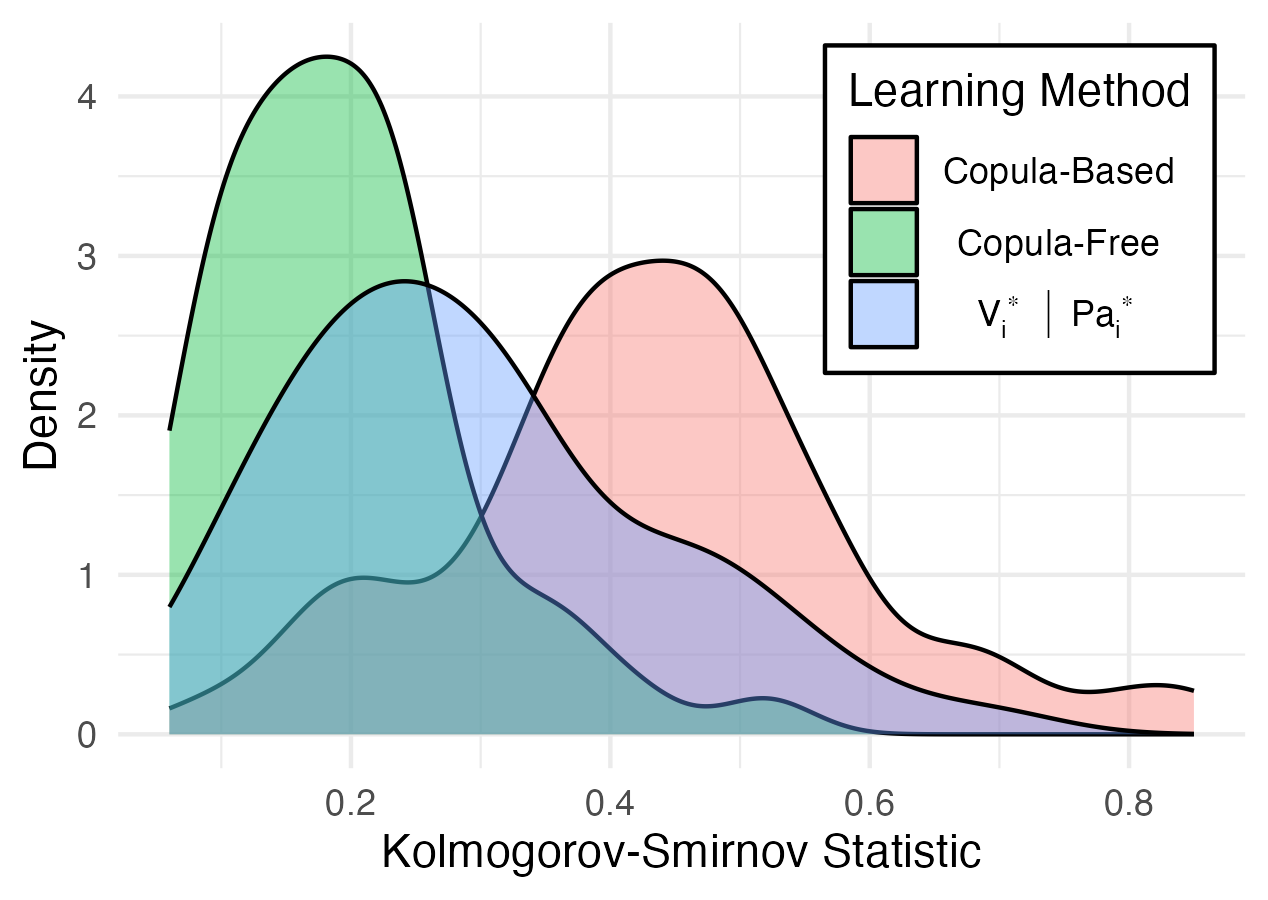}
        \vspace{-0.23in}
        \caption{Spread of KS statistics.}
        \label{fig:ks-statistics}
    \end{subfigure}
    \caption{(a) causal diagram for SeS-HELOC datasets; (b) spread of p-values for $H_0$ for cases A, B; (c) spread of Kolmogorov-Smirnov statistics for copula-free and copula-based learning for case B.}
    \label{fig:alg-3}
\end{figure}

\subsection{Alg.~\ref{algo:copula-free-learning} -- Copula-free Learning.} The final challenge we have is inferring effects of recourse actions when post-recourse stability does not hold, such as in the SeS-HELOC B dataset. In this case, we apply Alg.~\ref{algo:copula-free-learning} with $10^4$ pre-recourse samples, and assume that for all individuals with $\widehat Y = 0$ we also have post-recourse samples. For each individual in a separate test set of size $n_{\text{test}} = 50$, we compute 100 MC samples of $\widehat{Y}^{\mathrm{Alg.}\ref{algo:copula-free-learning}}_{R=r}(u^*)$ based on Alg.~\ref{algo:copula-free-learning}. As a comparison, we also compute 100 MC samples by assuming (i) Frank's copula holds true, and estimating the $\tau$ parameter, labeled $\widehat{Y}^F_{R=r}(u^*)$; (ii) by directly regressing $V_i^* \sim \Pa^*_i$, labeled $\widehat{Y}^{\text{exp}}_{R=r}(u^*)$. From the SCM itself, we compute 100 MC samples from the true underlying post-recourse distribution, labeled $\widehat{Y}_{R=r}(u^*)$. For each individual in the test set, we compute the Kolmogorov-Smirnov statistics $D$ of the MC samples $\widehat{Y}^{\mathrm{Alg.}\ref{algo:copula-free-learning}}_{R=r}(u^*)$, $\widehat{Y}^F_{R=r}(u^*)$, and $\widehat{Y}^\text{exp}_{R=r}(u^*)$ compared to $\widehat{Y}_{R=r}(u^*)$. The distributions of the statistics $D$ for the 50 individuals are shown in Fig.~\ref{fig:ks-statistics}. As the figure illustrates, the distance from the true distribution $\widehat{Y}_{R=r}(u^*)$ is smaller for $\widehat{Y}^{\mathrm{Alg.}\ref{algo:copula-free-learning}}_{R=r}(u^*)$ from Alg.~\ref{algo:copula-free-learning} than for $\widehat{Y}^\text{exp}_{R=r}(u^*)$, giving empirical verification of Thm.~\ref{thm:linear-recourse} in a non-parametric setting. Furthermore, $\widehat{Y}^\text{exp}_{R=r}(u^*)$ is still better than the copula-based estimator $\widehat{Y}^F_{R=r}(u^*)$, since the underlying conditions for copula-based learning from Def.~\ref{def:post-recourse-stability} are violated. 
%Nonetheless, Alg.~\ref{algo:copula-free-learning} can still be applied, provided sufficient amount of recourse data is available.

% \documentclass{standalone}
% \usepackage{pgfplots}
% \pgfplotsset{compat=1.17}

% \begin{document}

\section{Conclusion}
In this paper, we introduced a novel causal framework for combining observational and experimental data on the same individuals in the context of algorithmic recourse (Defs.~\ref{def:new-recourse}, \ref{def:recourse-data}, \ref{def:recourse-learning}). We discussed conditions that allow inference from observational data (Def.~\ref{def:post-recourse-stability} and Thm.~\ref{thm:prs-implies-msc}), and provided a copula-based sensitivity analysis for effects of recourse (Alg.~\ref{algo:frank-obs}). When recourse data is available, copula parameters can be inferred, and the copula model can be tested (Alg.~\ref{algo:frank-recourse}). Finally, when the copula model does not hold, a more general learning approach requiring more experimental recourse data is still feasible (Alg.~\ref{algo:copula-free-learning}). These claims were empirically validated on both real and semi-synthetic HELOC data.

\newpage
\bibliography{refs}

\appendix
\section{Conditional Likelihood of Frank's copula} \label{appendix:cond-mle}
In this appendix, we provide the expression for the conditiona likelihood of the bivariate Frank's copula. Frank's copula with parameter $\theta$ and variables $u, v$ has the following cumulative distribution function:
\begin{align}
    C_{\theta}(u,v) = -\frac{1}{\theta} \log \left[ 1 + \frac{(\exp{(-\theta u)} - 1)(\exp{(-\theta v)} - 1)}{\exp{(-\theta)} - 1} \right]. 
\end{align}
To compute the conditional likelihood $P(V = v \mid U = u)$, we need to find the joint density of $f_{\theta}(u, v)$ of $C_{\theta}(u,v)$. Note that
\begin{align}
    P(V = v \mid U = u) = \frac{P(V = v, U = u)}{P(U = u)} = P(V = v, U = u)
\end{align}
since $P(U = u) = 1$ is implied by the fact that the margins of a copula are Unif$[0, 1]$. The density of the copula $P(V = v, U = u)$ can obtained as
\begin{align} \label{eq:fuv-theta}
    f_{\theta}(u, v) = \frac{\partial^2}{\partial u \partial v} C_{\theta}(u, v) = \frac{-\theta\exp({-\theta(u+v)}) \cdot (\exp{(-\theta)} - 1)}{\left((\exp{(-\theta)} - 1) + (\exp(-\theta u) - 1) (\exp(-\theta v) - 1) \right)^2}.
\end{align}
The conditional log-likelihood given samples $\{u_i, v_i\}_{i=1:n}$, $\log L_{\theta}(\{v_i\}^n_{i=1} \mid \{u_i\}^n_{i=1})$ can thus be computed as
\begin{align}
     \sum_{i=1}^n \log \left [\frac{-\theta\exp({-\theta(u_i+v_i)}) \cdot (\exp{(-\theta)} - 1)}{\left((\exp{(-\theta)} - 1) + (\exp(-\theta u_i) - 1) (\exp(-\theta v_i) - 1) \right)^2} \right],
\end{align}
and conditional maximum likelihood estimation is performed as
\begin{align}
    \hat\theta = \argmax_{\theta} \log L_{\theta}(\{v_i\}^n_{i=1} \mid \{u_i\}^n_{i=1}). \label{eq:frank-cnd-mle}
\end{align}
We remark that the Kendall's $\tau$ parameter for Frank's copula is related to the $\theta$ parameter as:
\begin{align}
    \tau(\theta) = 1 + \frac{4}{\theta}[D_1(\theta) - 1]
\end{align}
where $D_1(\theta)$ is the Debye function $D_1(\theta) = \frac{1}{\theta} \int_{0}^\theta \frac{t}{e^t - 1} dt$.
\section{HELOC Causal Diagram} \label{appendix:HELOC-diagram}
The causal diagram of the HELOC dataset used in Sec.~\ref{sec:experiments} and the application of Alg.~\ref{algo:frank-obs} is given in Fig.~\ref{fig:heloc-full-dag}.
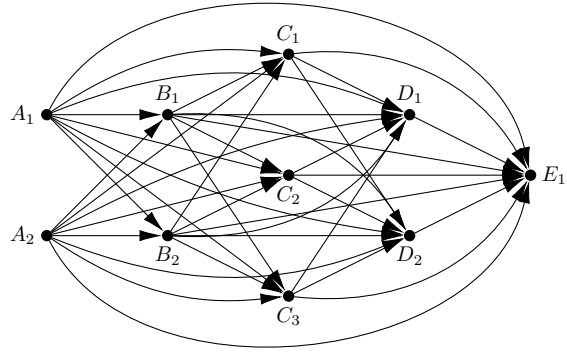
\begin{figure}
    \centering
    \begin{tikzpicture}[SCM,scale = 0.8]
        \pgfsetarrows{latex-latex};
            \newcommand{\xshift}{2}
            \newcommand{\yshift}{2}
            \node(a1) at (0 * \xshift, 0.5 * \yshift)[label=left:{$A_1$}, point];
            \node(a2) at (0 * \xshift, -0.5 * \yshift)[label=left:{$A_2$}, point];
            \node(b1) at (1 * \xshift, 0.5 * \yshift)[label=above:{$B_1$}, point];
            \node(b2) at (1 * \xshift, -0.5 * \yshift)[label=below:{$B_2$}, point];
            \node(c1) at (2 * \xshift, 1 * \yshift)[label=above:{$C_1$}, point];
            \node(c2) at (2 * \xshift, 0 * \yshift)[label=below:{$C_2$}, point];
            \node(c3) at (2 * \xshift, -1 * \yshift)[label=below:{$C_3$}, point];
            \node(d1) at (3 * \xshift, 0.5 * \yshift)[label=above:{$D_1$}, point];
            \node(d2) at (3 * \xshift, -0.5 * \yshift)[label=below:{$D_2$}, point];
            \node(e1) at (4 * \xshift, 0 * \yshift)[label=right:{$E_1$}, point];
            
            % G_i -> G_{i+1} edges
            \path (a1) edge (b1);
            \path (a1) edge (b2);
            \path (a2) edge (b1);
            \path (a2) edge (b2);
            \path (b1) edge (c1);
            \path (b1) edge (c2);
            \path (b1) edge (c3);
            \path (b2) edge (c1);
            \path (b2) edge (c2);
            \path (b2) edge (c3);
            \path (c1) edge (d1);
            \path (c2) edge (d1);
            \path (c3) edge (d1);
            \path (c1) edge (d2);
            \path (c2) edge (d2);
            \path (c3) edge (d2);
            \path (d1) edge (e1);
            \path (d2) edge (e1);

            % G_i -> G_{i+2} edges
            \path (a1) edge[bend left = 20] (c1);
            \path (a1) edge (c2);
            \path (a1) edge (c3);
            \path (a2) edge (c1);
            \path (a2) edge (c2);
            \path (a2) edge[bend left = -20] (c3);

            \path (b1) edge (d1);
            \path (b1) edge[bend left = 30] (d2);
            \path (b2) edge[bend left = -30] (d1);
            \path (b2) edge (d2);

            \path (c1) edge[bend left = 33] (e1);
            \path (c2) edge (e1);
            \path (c3) edge[bend left = -33] (e1);

            % G_i -> G_{i+3} edges
            \path (b1) edge (e1);
            \path (b2) edge (e1);

            \path (a1) edge[bend left = 22] (d1);
            \path (a1) edge[bend right = 10] (d2);
            \path (a2) edge[bend right = -10] (d1);
            \path (a2) edge[bend left = -22] (d2);

            % G_i -> G_{i+4} edges
            \path (a1) edge[bend left = 70] (e1);
            \path (a2) edge[bend left = -70] (e1);
        \end{tikzpicture}
    \caption{HELOC full causal diagram.}
    \label{fig:heloc-full-dag}
\end{figure}
\section{Thm.~\ref{thm:prs-implies-msc} Proof} \label{appendix:theorem-proofs}
\begin{proof}[Thm.~\ref{thm:prs-implies-msc} proof]
    To show that
    \begin{align}
        f_{V_i}(\pa_i, U_i) \overset{d}{=} f_{V_i}(\pa_i, U_i^*) \quad \forall\; i, \pa_i \text{ fixed},
    \end{align}
    we show that $U_i^*$ and $U_i$ follow the same distributions according to the theorem assumptions, i.e.,
    \begin{align}
        U_i^* \overset{d}{=} U_i. \label{eq:u-i-post-recourse-stab}
    \end{align}
    From Eq.~\ref{eq:u-i-post-recourse-stab}, the main claim follows, since $U_i^* \overset{d}{=} U_i \implies f(U_i^*) \overset{d}{=} f(U_i)$ for any function $f$. Now, we expand the distribution $P(U_i^*)$ as follows:
    \begin{align}
        P(U_i^* = u_i) &= P(U^{*(v)}_{i} = u^{(v)}_{i},  U^{*(f)}_{i} = u^{(f)}_{i}) \\
        &= P(U^{*(f)}_{i} = u^{(f)}_{i}) P(U^{*(v)}_{i} = u^{(v)}_{i} \mid  U^{*(f)}_{i} = u^{(f)}_{i}) \\
        &= P(U^{(f)}_{i} = u^{(f)}_{i}) P(U^{*(v)}_{i} = u^{(v)}_{i} \mid  U^{*(f)}_{i} = u^{(f)}_{i}) \quad (\text{Eq. }\ref{eq:u-1-invariant}) \\
        &= P(U^{(f)}_{i} = u^{(f)}_{i}) P(U^{(v)}_{i} = u^{(v)}_{i} \mid  U^{(f)}_{i} = u^{(f)}_{i}) \quad (\text{Eq. }\ref{eq:u-2-stable}) \\
        &= P(U^{(v)}_{i} = u^{(v)}_{i},  U^{(f)}_{i} = u^{(f)}_{i}) \\
        &= P(U_i = u_i).
    \end{align}
    Therefore, we have that $U_i^* \overset{d}{=} U_i$ and the claim follows.
\end{proof}
\section{Linear Recourse -- Thm.~\ref{thm:linear-recourse}} \label{appendix:linear-recourse}
In this section, we prove Thm.~\ref{thm:linear-recourse}. We first introduce the required notation and explain the theorem statement in more detail.

\subsection{Setting and Notation}\label{sec:setting}
We observe data from two linear models, referred to as \emph{pre-recourse} (subscript~$0$) and \emph{post-recourse}, sharing a common set of $n$ units. The data are split into a training set~(A) of size~$n_a$ and a test set~(B) of size~$n_b$, with $n = n_a + n_b$.

\paragraph{Pre-recourse model.}
The pre-recourse response $Y_0$ satisfies
\begin{equation}\label{eq:pre}
    Y_{0} = X_{0} \beta_0 + \eps_{0},
\end{equation}
where $X_{0} \in \mathbbm{R}^{n\times d}$ are covariates, $\beta_0 \in \mathbbm{R}^{d}$ is the coefficient vector, and $\eps_{0}$ is Gaussian noise.

\paragraph{Post-recourse model.}
The post-recourse response satisfies
\begin{equation}\label{eq:post}
    Y = X \beta + \eps,
\end{equation}
where $X \in \mathbbm{R}^{n \times d}$ are covariates, $\beta \in \mathbbm{R}^d$ is the coefficient vector, and $\eps$ is Gaussian noise.

\paragraph{Noise correlation.}
The post-recourse noise decomposes as
\begin{equation}\label{eq:noise-corr}
    \eps = \rho\,\eps_{0} + \sqrt{1-\rho^2}\,\beps,
\end{equation}
where $\rho \in (0,1)$ is the correlation between $\eps$ and $\eps_0$, and $\beps$ is an independent noise term.

\paragraph{Matrix notation.}
We write $X_0^a \in \mathbbm{R}^{n_a \times d}$, $X^a \in \mathbbm{R}^{n_a \times d}$ for the training covariate matrices, and similarly with superscript~$b$ for the test set. We write $\eps_0^a, \eps^a, \beps^a \in \mathbbm{R}^{n_a}$ for the corresponding noise vectors. Define the projection matrix
\begin{equation}
    P_0^a := X_0^a(X_0^{a\top}X_0^a)^{-1}X_0^{a\top}.
\end{equation}

\paragraph{Two-stage procedure.} \phantom{text} \vspace{0.1in} \\ 
\noindent \textbf{Stage~I.} Fit OLS on the pre-recourse training data:
    \begin{equation}\label{eq:stage1}
        \hbeta_0 = (X_0^{a\top} X_0^a)^{-1}X_0^{a\top} Y_0^a, \qquad \heps_0^a = Y_0^a - X_0^a \hbeta_0 = (I - P_0^a)\eps_0^a.
    \end{equation} \\
    \textbf{Stage~II.} Form the augmented design matrix $\bX^a = (X^a \;\; \heps_0^a) \in \mathbbm{R}^{n_a \times (d+1)}$ and fit OLS:
    \begin{equation}\label{eq:stage2}
        \hbeta_c = (\bX^{a\top}\bX^a)^{-1}\bX^{a\top}Y^a.
    \end{equation}

\paragraph{True augmented coefficient.}
Substituting~\eqref{eq:noise-corr} into~\eqref{eq:post} and using the decomposition $\eps_0 = \heps_0 + (\eps_0 - \heps_0)$, the post-recourse response can be written as
\begin{equation}\label{eq:augmented-model}
    Y = X \beta + \rho\,\heps_{0} + N,
\end{equation}
where the residual noise is
\begin{equation}\label{eq:N-def}
    N := \rho\,(\eps_{0} - \heps_{0}) + \sqrt{1-\rho^2}\,\beps.
\end{equation}
In vector form, $Y^a = \bX^a \beta_c + N^a$ with the true augmented coefficient
\begin{equation}\label{eq:betac}
    \beta_c := \begin{pmatrix}\beta \\ \rho\end{pmatrix} \in \mathbbm{R}^{d+1}.
\end{equation}

\paragraph{Prediction on test set.}
Let the training data $X_0^a, X^a, \eps_0^a, \beps^a$ be denoted by $\dt$.
For the test set, define $\heps_0^b = Y_0^b - X_0^b\hbeta_0$ and $\bX^b = (X^b \;\; \heps_0^b) \in \mathbbm{R}^{n_b \times (d+1)}$. The EMSPE of the two-stage estimator is
\begin{equation}
    R_{\mathrm{aug}} := \frac{1}{n_b}\,\ex\bigl[\|Y^b - \bX^b\hbeta_c\|^2 \mid \dt \bigr],
\end{equation}
and the baseline EMSPE using standard OLS on the post-recourse model alone is
\begin{equation}
    R_{\mathrm{std}} := \frac{1}{n_b}\,\ex\bigl[\|Y^b - X^b\hbeta\|^2 \mid \dt \bigr], \qquad \hbeta = (X^{a\top}X^a)^{-1}X^{a\top}Y^a.
\end{equation}
All expectations are conditional on $\dt$.

%% ============================================================
%% PART II: THEOREM AND PROOF
%% ============================================================

\subsection{Main Result}\label{sec:theorem}

\begin{theorem}[EMSPE reduction]\label{thm:linear-recourse-restate}
    Assuming $n_a \gg d$, we have
    \begin{equation}\label{eq:main-bound}
        R_{\mathrm{aug}} \;=\; (1 - \rho^2)\,\sigma^2 \;+\; \ordp{\sqrt{\frac{d}{n_a}}}.
    \end{equation}
    For $n_a$ sufficiently large, $R_{\mathrm{aug}} < R_{\mathrm{std}}$.
\end{theorem}

\begin{proof}[Thm.~\ref{thm:linear-recourse}]
Let $\dt := \{X_0^a, X^a, \eps_0^a, \beps^a \}$ denote the training data.
We start by expanding $R_{\mathrm{aug}}$. First, note that
    \begin{align}
        Y^b &= X^b\beta + \eps^b = X^b\beta + \rho \heps^b_0 + \rho (\eps^b_0 - \heps^b_0) + \sqrt{1-\rho^2} \beps^b \\
        &= \bX^b \beta_c + \underbrace{\rho (\eps^b_0 - \heps^b_0) + \sqrt{1-\rho^2} \beps^b}_{:= N^b} = \bX^b \beta_c + N^b.
    \end{align}
    Therefore, we have that
    \begin{align}
        \frac{1}{n_b} \ex [\| Y^b - \bX^b \hat \beta_c \|^2 \mid \dt] &= \frac{1}{n_b} \ex [ \| \bX^b (\beta_c - \hat\beta_c) + N^b \| ^ 2 \mid \dt] \\
        &= \underbrace{\frac{1}{n_b} \ex [ \| \bX^b (\beta_c - \hat\beta_c) \|^2 \mid \dt]}_{T_1} + \underbrace{\frac{1}{n_b} \ex [\| N^b \|^ 2 \mid \dt]}_{T_2} \\
        &\quad + \underbrace{\frac{2}{n_b} \ex [N^{b\top} \bX^b (\beta_c - \hat\beta_c) \mid \dt]}_{T_3}.
    \end{align}
    We bound the terms $T_1, T_2$ with high probability, specifically showing that 
    \begin{align}
        T_1 = \sigma^2 \ordp{\frac{d}{n_a}}
    \end{align}
    in Lem.~\ref{lemma:term1}, and showing that 
    \begin{align}
        T_2 = (1-\rho^2)\sigma^2 + \ordp{\frac{d}{n_a}}
    \end{align}
    in Lem.~\ref{lemma:term2}. After bounding $T_1, T_2$, for $T_3$ we have that
    \begin{align}
        T_3 &\overset{CS}{\leq} \frac{2}{n_b}  \sqrt{\ex [ \| N^b \|^2 \mid \dt] \cdot \ex [ \| \bX^b (\hbeta_c -\beta_c) \|^2 \mid \dt]} \\
        &= 2 \sqrt{ \frac{1}{n_b}\ex [ \| N^b \|^2 \mid \dt] \cdot \frac{1}{n_b}\ex [ \| \bX^b (\hbeta_c -\beta_c) \|^2 \mid \dt]} \\
        &= 2 \sqrt{T_1T_2} = \ordp{\sqrt{\frac{d}{n_a}}},
    \end{align}
    and the claim follows.
\end{proof}

\begin{lemma}[Term $T_2$] \label{lemma:term2}
    \begin{align}
        T_2 = (1-\rho^2)\sigma^2 + \ordp{\frac{d}{n_a}}.
    \end{align}
\end{lemma}
\begin{proof}[Lem.~\ref{lemma:term2}]
    Since $N^b = \rho (\eps^b_0 - \heps^b_0) + \sqrt{1-\rho^2} \beps^b$, we can write $\frac{1}{n_b} \ex [\| N^b \|^ 2 \mid \dt]$ as 
    \begin{align}
         \underbrace{\frac{\rho^2}{n_b}  \ex [\| \eps_0^b - \heps_0^b \|^ 2 \mid \dt]}_{T_{2(i)}} + \underbrace{\frac{(1-\rho^2)}{n_b}  \ex [\| \beps^b \|^ 2 \mid \dt]}_{T_{2(ii)}} + \underbrace{\frac{2\rho\sqrt{1-\rho^2}}{n_b} \ex [ (\eps_0^b - \heps_0^b)^\top \beps^b \mid \dt]}_{T_{2(iii)}}.
    \end{align}
    Note that $\beps^b$ is independent of $(\eps_0^b - \heps_0^b)$, so $T_{2(iii)}$ vanishes. $T_{2(ii)}$ is the expectation of a $\frac{(1-\rho^2)\sigma^2}{n_b} \chi_{n_b}^2$ distribution, equal to $(1-\rho^2) \sigma^2$. $T_{2(i)}$ satisfies
    \begin{align}
        \frac{\rho^2}{n_b}  \ex [\| \eps_0^b - \heps_0^b \|^ 2 \mid \dt] &= \frac{\rho^2}{n_b}  \ex [\| X_0^b (\beta_0 - \hbeta_0) \|^ 2 \mid \dt] \\ 
        &= \rho^2 (\beta_0 - \hbeta_0)^\top \Sigma_0 (\beta_0 - \hbeta_0) \\
        &= \rho^2 \eps_0^{a\top} X_0^a(X_0^{a\top} X_0^a)^{-1} \Sigma_0^{1/2} \Sigma_0^{1/2}  (X_0^{a\top} X_0^a)^{-1} X_0^{a\top} \eps_0^a \\
        &= \rho^2 ((\tilde X_0^{a\top} \tilde X_0^a)^{-1/2} \tilde X_0^{a\top} \eps_0^{a})^\top (\tilde X_0^{a\top} \tilde X_0^a)^{-1} (\tilde X_0^{a\top} \tilde X_0^a)^{-1/2} \tilde X_0^{a\top} \eps_0^{a},
    \end{align}
    where $\tilde X_0^a := X_0^a \Sigma_0^{-1/2}$. Now, note that $(\tilde X_0^{a\top} \tilde X_0^a)^{-1}$ follows an inverse-Wishart $IW(I_d, n_a)$ distribution, while $(\tilde X_0^{a\top} \tilde X_0^a)^{-1/2} \tilde X_0^{a\top} \eps_0^{a}$ is a Gaussian with covariance $\sigma^2 I_{d}$, so by recognizing Hotelling's $T^2$ distribution we have
    \begin{align}
        T_{2(i)} \sim \rho^2 \sigma^2 \frac{1}{n_a} T^2(d, n_a) \sim \rho^2 \sigma^2 \frac{d}{n_a - d +1} F_{d, n_a - d +1}
    \end{align}
    which is $\rho^2 \sigma^2 \ordp{\frac{d}{n_a}}$ as $n_a \to \infty$. The claim follows by putting together terms $T_{2(i)}$ and $T_{2(ii)}$ together.
\end{proof}

\begin{lemma}[Term $T_1$] \label{lemma:term1}
    Term $T_1$ satisfies
    \begin{align}
        T_1 = \ordp{\frac{d}{n_a}}.
    \end{align}
\end{lemma}
\begin{proof}
    Note that
    \begin{align}
        T_1 &= (\hbeta_c - \beta_c)^\top  \ex[ \frac{1}{n_b} \bX^{b\top} \bX^b \mid \dt] (\hbeta_c - \beta_c) \\
        &= \|\hbeta_c - \beta_c\|_{\tSig}^2
    \end{align}
    where $\|x\|_\Sigma := \sqrt{x^\top \Sigma x}$, and the matrix $\tSig$ is the covariance matrix of $\bX^b = \begin{pmatrix}
        X^b & \heps_0^b
    \end{pmatrix}$,
    given by
    \begin{align}
        \begin{pmatrix}
            \Sigma & \Sigma_{X, X_0} (\beta_0 - \hbeta_0) \\
            (\beta_0 - \hbeta_0)^\top \Sigma_{X_0, X} & \sigma^2 + \| \beta_0 - \hbeta_0\|^2_{\Sigma_0}
        \end{pmatrix}.
    \end{align}
    Since $\hbeta_c - \beta_c = \frac{1}{n_a} \hat\Sigma^{-1}_{\bX^a} \bX^{a\top} N^a$, we have that
    \begin{align}
        \|\hbeta_c - \beta_c\|_{\tSig}^2 &= \frac{1}{n_a^2} N^{a\top} \bX^a \hat\Sigma^{-1}_{\bX^a} \tSig \hat\Sigma^{-1}_{\bX^a} \bX^{a\top} N^a \\
        &= \| \frac{1}{n_a} \tSig^{1/2} \hat\Sigma^{-1/2}_{\bX^a} \hat\Sigma^{-1/2}_{\bX^a} \bX^{a\top} N^a \|^2 \\
        &\leq \| \tSig^{1/2} \hat\Sigma^{-1/2}_{\bX^a} \|^2 \| \frac{1}{n_a}\hat\Sigma^{-1/2}_{\bX^a} \bX^{a\top} N^a \|^2 \\
        &= \underbrace{\| \tSig^{1/2}\hat\Sigma^{-1}_{\bX^a} \tSig^{1/2} \|}_{\mathcal{K}}  \underbrace{\| \frac{1}{n_a} \hat\Sigma^{-1/2}_{\bX^a} \bX^{a\top} N^a \|^2}_{\mathcal{F}},
    \end{align}
    where the last line uses the fact that $\|A^\top \|^2 = \| A^\top A \|$, applied to $A =  \hat\Sigma^{-1/2}_{\bX^a} \tSig^{1/2}$. In Lem.~\ref{lem:cov-mis} we show that $\mathcal{K} = 1 + \ordp{\sqrt{\frac{d}{n_a}}}$, while in Lem.~\ref{lem:fixed-design} we show that $\mathcal{F} = \ordp{\frac{d}{n_a}}$, which together imply that $T_1 = \ordp{\frac{d}{n_a}}$.
\end{proof}

\begin{lemma}[Covariance Mismatch] \label{lem:cov-mis}
    \begin{align}
        \mathcal{K} = 1 + \ordp{\sqrt{\frac{d}{n_a}}}.
    \end{align}
\end{lemma}
\begin{proof}
    We define the matrix $\bSig$ as
    \begin{align}
        \bSig = \begin{pmatrix}
            \Sigma & 0 \\
            0 & \sigma^2 + \| \beta_0 - \hbeta_0\|^2_{\Sigma_0}
        \end{pmatrix},
    \end{align}
    and denote $\tau^2 := \sigma^2 + \| \beta_0 - \hbeta_0\|^2_{\Sigma_0}$, and $E := \tSig - \bSig$. Note that we have
    \begin{align}
        \mathcal{K} &= \| {\tSig}^{1/2}\hSigma_{\bX^a}^{-1}{\tSig}^{1/2} \| = \| \bigl({\tSig}^{1/2}{\bSig}^{-1/2}\bigr)\bigl({\bSig}^{1/2}\hSigma_{\bX}^{-1}{\bSig}^{1/2}\bigr)\bigl({\bSig}^{-1/2}{\tSig}^{1/2}\bigr) \| \\
        &\leq \| {\tSig}^{1/2}{\bSig}^{-1/2} \| ^ 2 \underbrace{\| {\bSig}^{1/2}\hSigma_{\bX}^{-1}{\bSig}^{1/2} \| }_{\mathcal{K}_0}.
    \end{align} 
    Further, we have that
    \begin{align}
        \| {\tSig}^{1/2}{\bSig}^{-1/2} \| ^ 2 &= \| {\bSig}^{-1/2} {\tSig} {\bSig}^{-1/2} \| \\
        &= \| I + {\bSig}^{-1/2} E {\bSig}^{-1/2} \| \\
        &\leq 1 + \| {\bSig}^{-1/2} E {\bSig}^{-1/2}\| 
    \end{align}
    and the matrix ${\bSig}^{-1/2} E {\bSig}^{-1/2}$ has only off-diagonal non-zero entries 
    $\frac{1}{\tau}\Sigma^{-1/2}\Sigma_{X,X_0}(\beta_0 - \hbeta_0)$, for which we can write
    \begin{align}
        \frac{\|\Sigma^{-1/2}\Sigma_{X,X_0}(\beta_0 - \hbeta_0)\|}{\tau} &=  \frac{\|\Sigma^{-1/2}\Sigma_{X,X_0} \Sigma_0^{-1/2} \Sigma_0^{1/2} (\beta_0 - \hbeta_0)\|}{\tau} \\
        &\leq \frac{\|\Sigma^{-1/2}\Sigma_{X,X_0} \Sigma_0^{-1/2}\|}{\tau} \cdot\|\beta_0 - \hbeta_0\|_{\Sigma_0} \\
        &\leq \frac{\|\Sigma^{-1/2}\Sigma_{X,X_0} \Sigma_0^{-1/2}\|}{\sigma} \cdot \|\beta_0 - \hbeta_0\|_{\Sigma_0}  = \ordp{\sqrt{\frac{d}{n_a}}}
    \end{align}
    since $\frac{\|\Sigma^{-1/2}\Sigma_{X,X_0} \Sigma_0^{-1/2}\|}{\sigma}$ is a constant and $\|\beta_0 - \hbeta_0\|_{\Sigma_0}$ is $\ordp{\sqrt{\frac{d}{n_a}}}$ as shown in the proof of Lem.~\ref{lemma:term2} (term $T_{2(i)}$). Therefore, it only remains to show that $\mathcal{K}_0 = 1 + \ordp{\sqrt{\frac{d}{n_a}}}$, from which follows that that $\mathcal{K} = 1+\ordp{\sqrt{\frac{d}{n_a}}}$. 

    For bounding $\mathcal{K}_0$, let $\hat S := \frac{1}{n_a} X^{a\top} \heps_0^a, \hat\tau^2 = \frac{1}{n_a} \| \heps_0^a \|^2$. Using this notation, we can write
    \begin{align}
        {\bSig}^{1/2}\hSigma_{\bX}^{-1}{\bSig}^{1/2} = (I + \Delta)^{-1}, \quad \Delta := \begin{pmatrix}
            \Sigma^{-1/2} (\hSigma_{X^a} - \Sigma) \Sigma^{-1/2} & \Sigma^{-1/2} \hat S / \tau \\
            \hat S^\top \Sigma^{-1/2} / \tau & (\hat \tau^2 - \tau^2) / \tau^2
        \end{pmatrix}.
    \end{align}
    If $\| \Delta \| < 1,$ then 
    \begin{align}
        \mathcal{K}_0 = \| (I+\Delta)^{-1}\| \leq \frac{1}{ \lambda_{min} (I + \Delta)} \leq (1 - \| \Delta\|)^{-1},
    \end{align}
    so we want to show that $\| \Delta \|$ is $\ordp{\sqrt{\frac{d}{n_a}}}$, which implies $\mathcal{K}_0 = \ordp{\sqrt{\frac{d}{n_a}}}$. Write $\Delta = \bigl(\begin{smallmatrix}A & B \\ B^\top & D\end{smallmatrix}\bigr)$, and by using the fact that $\| \Delta \| \leq \max (\| A \| + \| B \|, \|B \| + |D|)$, it is enough to bound the norm of each of the submatrices.

    \emph{Block $(1, 1)$ -- matrix $A$.} For the matrix $A = \Sigma^{-1/2} (\hSigma_{X^a} - \Sigma) \Sigma^{-1/2}$, we have that
    \begin{align}
        \| \Sigma^{-1/2} (\hSigma_{X^a} - \Sigma) \Sigma^{-1/2} \| \leq \|  \Sigma^{-1}\| \cdot \| \hSigma_{X^a} - \Sigma \|  = \ordp{\sqrt{\frac{d}{n_a}}}
    \end{align}
    since $\| \hSigma_{X^a} - \Sigma \| = \ordp{\sqrt{\frac{d}{n_a}}}$ using a Wishart concentration inequality \citep{vershynin2020high}.

    \emph{Block $(2, 2)$ -- scalar $D$.} Note that $\hat\tau^2 = \frac{1}{n_a} \eps_0^{a\top} (I-P_0^a) \eps_0^a$, which conditional on $X_0^a$ follows a $\frac{\sigma^2}{n_a} \chi^2_{n_a-d}$ distribution, and this equals $\sigma^2 + \ordp{\frac{1}{\sqrt{n_a}}}$ since $\frac{\chi^2_k}{k} = 1 + \ordp{\frac{1}{\sqrt{k}}}$ \citep{laurent2000adaptive}. From before, we obtained that $\tau^2 = \sigma^2 + \| \hbeta_0-\beta_0\|^2_{\Sigma_0} = \sigma^2 + \ordp{{\frac{d}{n_a}}}$. Putting everything together
    \begin{align}
        \frac{|\hat\tau^2 - \tau^2|}{\tau^2} \leq \frac{1}{\sigma^2} |\hat\tau^2 - \tau^2| = \ordp{\sqrt{\frac{1}{n_a}}}.
    \end{align}

    \emph{Off-diagonal block $(1, 2)$ -- vector $B$.} Finally, we need to bound $\| \Sigma^{-1/2} \hat S / \tau \|$. Let $\Sigma_{X, X_0} := Cov(X_i^a, X_{0,i}^a)$. Note that
    \begin{align}
        X_i^a \mid X_{0,i}^a \sim N( \Sigma_{X, X_0} \Sigma_0^{-1} X_{0, i}, \Sigma_{X\mid X_0}),
    \end{align}
    where $\Sigma_{X\mid X_0} = \Sigma - \Sigma_{X,X_0}\Sigma_0^{-1}\Sigma_{X_0,X}$. Therefore, we have that
    \begin{align}
        \frac{1}{n_a} X^{a\top} \heps_0^a = \frac{1}{n_a} \sum_{i=1}^{n_a} \heps_{0,i}^a X_i^{a},
    \end{align}
    showing that conditional of $X_0^a, \eps_0^a$, $\hat S$ is a Gaussian, with a mean
    \begin{align}
        \ex [\hat S \mid X_0^a, \eps_0^a] = \frac{1}{n_a} \Sigma_{X,X_0}\Sigma_0^{-1} X_0^{a\top} (I-P_{0}) \eps_0^a = 0,
    \end{align}
    since $I-P_{0}$ projects onto the orthogonal complement of the column space of $X_0^a$. The covariance of $\hat S$ (conditional on $X_0^a, \eps_0^a$) is 
    \begin{align}
        \var\left(\frac{1}{n_a} \sum_{i=1}^{n_a} \heps_{0,i}^a X_i^{a} \mid X_0^a, \eps_0^a\right) &= \frac{1}{n_a^2} \sum_{i=1}^{n_a} {(\heps_{0,i}^a)}^2 \var\left(X_i^{a} \mid X_{0,i}^a\right) \\
        &= \Sigma_{X\mid X_0} \frac{1}{n_a^2} \| \heps_0^a \|^2 = \frac{\hat\tau^2}{n_a} \Sigma_{X\mid X_0}.
    \end{align}
    Putting together, $\tilde S := \frac{1}{\tau}\Sigma^{-1/2} \hat S \mid X_0 \sim N(0, \frac{\hat\tau^2}{\tau^2 n_a} \Sigma^{-1/2} \Sigma_{X\mid X_0}\Sigma^{-1/2})$. Therefore, 
    \begin{align} \label{eq:tilde-S}
        \| \tilde S \|^2 \overset{d}{=} \frac{\hat\tau^2}{\tau^2 n_a} \sum_{i=1}^d \lambda_i Z_i^2,
    \end{align}
    where $\lambda_i$ are the eigenvalues of $\Sigma^{-1/2} \Sigma_{X\mid X_0}\Sigma^{-1/2}$ and $Z_i$ are i.i.d. Gaussians. Note that $\Sigma_{X\mid X_0} = \Sigma - \Sigma_{X,X_0}\Sigma_0^{-1}\Sigma_{X_0,X} \preccurlyeq \Sigma$, which implies $ \Sigma^{-1/2} \Sigma_{X\mid X_0}\Sigma^{-1/2} \preccurlyeq I_d$, meaning that all eigenvalues $\lambda_i$ in Eq.~\ref{eq:tilde-S} are smaller than $1$, so that $\| \tilde S \|^2$ is stochastically smaller than $\frac{\hat\tau^2}{\tau^2 n_a} \chi^2_d$. Since $\frac{\hat\tau^2}{\tau} = 1 + \ordp{\sqrt{\frac{1}{n_a}}}$ and $\chi^2_d = d + \ordp{\sqrt{d}}$, we have that $\| \tilde S\|^2 = \ordp{\frac{d}{n_a}}$, meaning that $\| S \| = \ordp{\sqrt{\frac{d}{n_a}}}$, completing the proof.
\end{proof}

\begin{lemma}[Fixed Design Loss] \label{lem:fixed-design}
    The fixed design loss, denoted by $\mathcal{F}$, satisfies
    \begin{align}
        \| \frac{1}{n_a} \hat\Sigma^{-1/2}_{\bX^a} \bX^{a\top} N^a \|^2 = \ordp{?}
    \end{align}
\end{lemma}
\begin{proof}
    Note that
    \begin{align}
        \| \frac{1}{n_a} \hat\Sigma^{-1/2}_{\bX^a} \bX^{a\top} N^a \|^2  = \frac{1}{n_a} N^{a\top} \bX^a (\bX^{a\top}\bX^a)^{-1} \bX^{a\top} N^a = \frac{1}{n_a} N^{a\top} H_a N^a,  
    \end{align}
    where $H_a$ is the projection onto the column space of $\bX^a$. By definition $N^a = \rho P^a_0 \eps_0^a + \sqrt{1-\rho^2} \beps^a$, meaning that
    \begin{align}
        \mathcal{F} = \underbrace{\frac{\rho^2}{n_a} \eps_0^{a\top} P^a_0 H^a P^a_0 \eps_0^a}_{\mathcal{F}_1} + \underbrace{\frac{1-\rho^2}{n_a} \beps^{a\top} H^a \beps^a}_{\mathcal{F}_2} + \underbrace{\frac{\rho\sqrt{1-\rho^2}}{n_a}  \eps_0^{a\top} P^a_0 H^a \beps^a}_{\mathcal{F}_3}.
    \end{align}
    Conditional on $X^a, X_0^a, \eps_0^a$ (meaning $\bX^a$ fixed), note that $\mathcal{F}_2$ follows a $\frac{1-\rho^2}{n_a} \chi^2_{d+1}$ distribution, meaning it is $\ordp{\frac{d}{n_a}}$.
    For term $\mathcal{F}_1$, we note that conditional on $X^a, X_0^a,\heps_0^a$, both $H^a, P_0^a$ are fixed. Since $\heps_0^a = (I-P_0^a)\eps_0^a$ is independent of $P_0^a \eps_0^a$, the conditioning on $\heps_0^a$ does not affect $P_0^a \eps_0^a$, which follows a distribution $N(0, \sigma^2 P_0^a)$, meaning that $H^a P_0^a \eps_0^a$ follows a $N(0, \sigma^2 P_0^a H^a)$ distribution. Therefore, we conclude that 
    \begin{align}
        \mathcal{F}_1 \overset{d}{=} \frac{\rho^2\sigma^2}{n_a} \sum_{i=1}^{d+1} \lambda_i Z_i^2,
    \end{align}
    where $\lambda_i \in [0, 1]$ are the top $d+1$ eigenvalues of $P_0^a H^a$, and $Z_i$ are independent. Therefore, $\mathcal{F}_1$ is $\ordp{\frac{d}{n_a}}$. Finally, for the term $\mathcal{F}_3$, we condition on $X^a, X_0^a,\heps_0^a$, meaning that $H^a, P_0^a$ are fixed. Denote by $w := \frac{1}{n_a}H^a P_0^a \eps_0^a$. Conditional on $w$, we have that $w^\top \beps^a \sim N(0, \sigma^2 \| w \| ^2)$, so 
    \begin{align} \label{eq:f3-tail}
        P ( | w^\top \beps^a| > \sigma \|w\| \sqrt{2\log{\frac{1}{\delta}}} ) \leq 2\delta.
    \end{align}
    As we argued before, $\| H^a P_0^a \eps_0^a \|^2$ is stochastically dominated by a $\chi^2_{d+1}$, so that $\| w\| $ is $\ordp{\frac{\sqrt{d}}{n_a}}$, and by using this fact with Eq.~\ref{eq:f3-tail} and applying a union bound, we have that $\mathcal{F}_3$ is $\ordp{\frac{\sqrt{d}}{n_a}}$. Putting everything together, we obtain that $\mathcal{F}$ is $\ordp{\frac{d}{n_a}}$.
\end{proof}

\end{document}